\documentclass[11pt]{article}
\usepackage[margin=1in]{geometry}

\usepackage[utf8]{inputenc}
\usepackage[backref,colorlinks,citecolor=blue,bookmarks=true]{hyperref}
\usepackage{mathtools, amssymb, amsthm, bbm}
\usepackage[capitalize]{cleveref}
\usepackage{enumerate}
\usepackage{autonum}
\usepackage{todonotes}

\theoremstyle{plain}
\newtheorem{theorem}{Theorem}[section]
\newtheorem{lemma}[theorem]{Lemma}
\newtheorem{corollary}[theorem]{Corollary}

\theoremstyle{definition}
\newtheorem{definition}[theorem]{Definition}

\theoremstyle{remark}

\numberwithin{equation}{section}

\newcommand*{\N}{{\mathbb{N}}}

\newcommand*{\R}{{\mathbb{R}}}
\newcommand*{\calC}{{\mathcal{C}}}

\newcommand*{\calD}{{\mathcal{D}}}

\newcommand*{\calN}{{\mathcal{N}}}
\newcommand*{\calX}{{\mathcal{X}}}

\newcommand*{\calR}{{\mathcal{R}}}
\newcommand*{\calI}{{\mathcal{I}}}

\newcommand*{\Dx}{D_{\calX}}
\let\eps\epsilon
\let\phi\varphi
\DeclareMathOperator*{\pr}{\mathbb{P}}
\DeclareMathOperator*{\ex}{\mathbb{E}}
\DeclareMathOperator{\var}{Var}

\DeclareMathOperator{\unif}{Unif}

\DeclareMathOperator*{\argmin}{arg\,min}

\DeclarePairedDelimiter{\inn}{\langle}{\rangle}

\renewcommand{\th}{^\text{th}}
\let\hat\widehat
\DeclareMathOperator{\poly}{poly}

\DeclareMathOperator{\sgn}{sign}

\DeclareMathOperator{\ind}{\mathbbm{1}}
\newcommand*\diff{\mathop{}\!\mathrm{d}}
\newcommand{\opt}{\mathsf{opt}}
\newcommand{\cube}[1]{\{\pm 1\}^{#1}}
\newcommand{\wt}[1]{\widetilde{#1}}
\newcommand{\acz}{\mathsf{AC}^0}
\newcommand{\ignore}[1]{} 

\title{A Moment-Matching Approach to Testable Learning and a New Characterization of Rademacher Complexity}
\author{Aravind Gollakota\thanks{\texttt{aravindg@cs.utexas.edu}. Supported by NSF award AF-1909204 and the NSF AI Institute for Foundations of Machine Learning (IFML).} \\ UT Austin
	\and Adam R. Klivans\thanks{\texttt{klivans@cs.utexas.edu}. Supported by NSF award AF-1909204 and the NSF AI Institute for Foundations of Machine Learning (IFML).} \\
	UT Austin
	\and Pravesh K. Kothari\thanks{\texttt{praveshk@cs.cmu.edu}. Supported by NSF CAREER Award \#2047933, NSF \#2211971, an Alfred P. Sloan Fellowship, and a Google Research Scholar Award.} \\
	CMU
}
\date{November 19, 2022}

\begin{document}

\maketitle

\begin{abstract}
A remarkable recent paper by Rubinfeld and Vasilyan \cite{rubinfeld2022testing} initiated the study of \emph{testable learning}, where the goal is to replace hard-to-verify distributional assumptions (such as Gaussianity) with efficiently testable ones and to require that the learner succeed whenever the unknown distribution passes the corresponding test. In this model, they gave an efficient algorithm for learning halfspaces under testable assumptions that are provably satisfied by Gaussians.

In this paper we give a powerful new approach for developing algorithms for testable learning using tools from moment matching and metric distances in probability.  We obtain efficient testable learners for any concept class that admits low-degree \emph{sandwiching polynomials}, capturing most important examples for which we have ordinary agnostic learners.  We recover the results of Rubinfeld and Vasilyan as a corollary of our techniques while achieving improved, near-optimal sample complexity bounds for a broad range of concept classes and distributions.

Surprisingly, we show that the information-theoretic sample complexity of testable learning is tightly characterized by the Rademacher complexity of the concept class, one of the most well-studied measures in statistical learning theory. In particular, uniform convergence is necessary and sufficient for testable learning.  This leads to a fundamental separation from (ordinary) distribution-specific agnostic learning, where uniform convergence is sufficient but not necessary.
\end{abstract}

\section{Introduction} \label{sec:intro}

In the fundamental model of agnostic learning \cite{kearns1992toward,vapnik1998statistical}, a learner tries to output the best-fitting function from a concept class $\calC$ with respect to an unknown labeled distribution $\calD$ in the following sense: given sufficiently many labeled examples, with high probability it must produce a hypothesis with error at most $\opt(\calC, \calD) + \eps$ over $\calD$, where $\opt(\calC, \calD)$ denotes the optimal error achievable over $\calD$ by any concept in $\calC$. No assumptions are made on the labels.

Agnostic learning is known to be computationally intractable for even the simplest function classes without making assumptions on the marginal \cite{kearns1992toward,kearns1994cryptographic,klivans2009cryptographic,guruswami2009hardness,feldman2012agnostic,daniely2016complexity,daniely2016dnf}.  There is now a substantial literature of efficient agnostic learning algorithms under various distributional assumptions, the most common being that the marginal is Gaussian or $\unif\cube{d}$ (see e.g.\ \cite{linial1993constant,bshouty1996fourier,kalai2008agnostically,klivans2008learning,kane2011gaussian}).  The problem of directly verifying this distributional assumption from samples, however, is often computationally infeasible (such as for $\unif\cube{d}$) or fundamentally ill-posed (as for $\calN(0, I_d)$\footnote{To see why this is the case even when $d = 1$, fix any finite sample size $m$, and consider a (random) discrete distribution $\hat{D}$ that is uniform on $\Omega(m^2)$ points drawn from $\calN(0,1)$. This distribution has TV distance $1$ from $\calN(0,1)$ (since it is discrete), yet with high probability a sample of size $m$ drawn from $\hat{D}$ will be duplicate-free and distributed exactly as a sample of size $m$ drawn from $\calN(0,1)$.}). 

Since in the agnostic model we make no assumptions on the labels, we have no a priori estimate of $\opt(\calC, \calD)$, the error of the best-fitting classifier.  Thus, a major (and often overlooked) issue with the agnostic learning model is that {\em it is unclear how to verify that the agnostic learner has actually succeeded.}  Note that while we can estimate the true error of the output hypothesis on a hold-out set (a.k.a.\ validation), we do not know its relationship to $\opt(\calC, \calD)$.  

With this motivation in mind, very recent work of Rubinfeld and Vasilyan \cite{rubinfeld2022testing} introduced the elegant model of testable agnostic learning, or just testable learning for short. In this model, no assumptions are made on $\calD$, but there is a tester responsible for verifying whether the unknown marginal is suitably well-behaved. Whenever the tester accepts, the learner must succeed at producing a hypothesis with error at most $\opt(\calC, \calD) + \eps$ (with high probability). And to ensure nontriviality, whenever the unknown marginal is indeed a certain well-behaved target marginal $\Dx$, the tester must accept (with high probability). We say the class $\calC$ is testably learnable with respect to a target marginal $\Dx$ if there is a tester-learner pair meeting these conditions (see \cref{def:testable-learning}).

In this model, \cite{rubinfeld2022testing} showed that halfspaces can be testably learned with respect to Gaussians in time and sample complexity $d^{\wt{O}(1/\eps^4)}$. Their proof involves checking that the low-degree moments of the unknown marginal are close to those of a Gaussian. They show that this implies concentration and anticoncentration properties of the unknown marginal and further prove that any distribution (including the empirical distribution on samples) that satisfies such properties admits low-degree polynomial approximators for halfspaces. Their analysis, however, is catered specifically to the case of halfspaces and Gaussian marginals. 

We note that in independent and concurrent work, Rubinfeld and Vasilyan~\cite{rubinfeld2022personal} have found a testable learning algorithm for halfspaces over the uniform distribution on the hypercube with the same $d^{\tilde{O}(1/\epsilon^4)}$ sample complexity as in the Gaussian case. As we'll discuss, their techniques are quite different from ours.

\subsection{Our results}
Our main algorithmic contribution is a general framework that yields efficient testable learning algorithms for broad classes of functions and distributions (both continuous and discrete). As we discuss in more detail below, our framework departs from the focus on constructing low-degree polynomial approximations with respect to absolute loss as in~\cite{rubinfeld2022testing}, which appears hard to extend to classes beyond a single halfspace. Instead, we rely on a new connection to a stronger type of approximator --- sandwiching polynomials --- that arises naturally in constructing pseudorandom generators for classes of Boolean functions.

As it turns out, many interesting and well-studied concept classes admit both sandwiching approximators and ordinary low-degree polynomial approximators of essentially the same degree, even though sandwiching is a formally stronger notion.  As a result, we derive testable learning algorithms for halfspaces and more generally arbitrary functions of a bounded number of halfspaces with respect to any fixed strongly logconcave distribution. For the uniform distribution on the hypercube, we obtain algorithms for halfspaces, degree-2 PTFs, and constant-depth circuits. For each of these applications, our running times and sample complexity guarantees match the best known results for ordinary agnostic learning, thus showing that testable learning can often be achieved at no additional cost  (see \cref{thm:learn-kwise} and \cref{thm:learn-hs} for precise statements). 

In particular, for the special case of testably learning a single halfspace with respect to the Gaussian, our results improve the $d^{\wt{O}(1/\epsilon^4)}$ running time and sample complexity guarantee shown in~\cite{rubinfeld2022testing} to $d^{\wt{O}(1/\epsilon^2)}$, matching the best known (and conditionally optimal) results for ordinary agnostic learning. Moreover, our analysis extends to a broad family of distributions including strongly logconcave distributions.

We now describe our results and discuss our techniques in more detail.

\paragraph{Sandwiching polynomial approximation.} Our starting point is a relationship between testable learning and a certain stronger notion of polynomial approximation that arises naturally in building pseudorandom generators for Boolean function classes. Specifically, a concept class $\calC$ admits \emph{sandwiching} approximations of degree $k$ and error $\epsilon$ on $\Dx$ if for every function $f \in \calC$, there exist two degree-$k$ polynomials $p_l, p_u$ such that \emph{for every $x$}, $p_l(x) \leq f(x) \leq p_u(x)$, and moreover $\ex_{\Dx}[f - p_l], \ex_{\Dx}[p_u - f] \leq \eps$. Observe that this is a stronger requirement than the existence of approximating polynomials for $\calC$ with respect to absolute loss, which only requires that for every $f$, there be a degree $k$ polynomial $p$ such that $\ex_{\Dx}[|p-f|] \leq \eps$.

The main theorem underlying our framework shows that unlike the existence of polynomial approximators with respect to absolute loss, the existence of sandwiching approximations universally translates into testable learning algorithms:

\begin{theorem}[Testable learning using approximate moment matching; see \cref{thm:main-algorithm}]\label{thm:main-algorithm-intro}
Let $\Dx$ be a distribution on $\calX$, and let $\calC$ be a concept class mapping $\calX$ to $\cube{}$. Let $k \in \N, \delta \in \R_+$ be degree and slack parameters, and let $\eps > 0$ be the error parameter. Suppose that every $f \in \calC$ admits degree-$k$ $\eps$-sandwiching polynomials $p_l \leq f \leq p_u$ w.r.t.\ $\Dx$ such that $\|p_l\|_1, \|p_u\|_1 \leq \eps/\delta$, where $\|p_u\|_1$ (resp.\ $\|p_l\|_1$) refers to the $\ell_1$ norm of the coefficients of $p_u$ (resp.\ $p_l$). Suppose also that with high probability over a sample of size $d^{O(k)}$, the empirical moments of degree at most $k$ of $\Dx$ are within $\delta$ of their true moments. Then $\calC$ can be testably learned w.r.t.\ $\Dx$ up to excess error $O(\eps)$ in sample and time complexity $d^{O(k)}/\poly(\eps)$.
\end{theorem}

Theorem~\ref{thm:main-algorithm-intro} relies on a simple tester: verify that the empirical moments of degree at most $k$ are close enough to the those of $\Dx$. The correctness of the tester relies on the claim that sandwiching polynomials for $\calC$ under $\Dx$ are also sandwiching polynomials for $\calC$ under the uniform distribution $\hat{D}$ on a large enough sample from $\Dx$, with an additional error that scales proportional to the $\ell_1$ norm of the coefficients of the polynomial approximators. Thus, whenever we have sandwiching approximators with appropriate bounds on the $\ell_1$ norm of the coefficient vectors, our testable learner can simply use the now-standard degree-$k$ (absolute-loss) polynomial regression algorithm of \cite{kalai2008agnostically}.

 The work of \cite{kalai2008agnostically} showed that the existence of low-degree (not necessarily sandwiching) polynomial approximators with respect to absolute loss suffices for {\em ordinary} agnostic learning. Our theorem on sandwiching polynomial approximators can be thought of as the natural counterpart to their condition but for \emph{testable} agnostic learning.


Our main task now reduces to constructing sandwiching polynomials with sufficiently small coefficients. Since proofs of existence of sandwiching polynomials are sometimes nonconstructive (e.g., for constant-depth circuits over the hypercube, where the existence of such polynomials follows from LP duality~\cite{bazzi2009polylogarithmic}), we require new techniques to prove bounds on the $\ell_1$ norm of the coefficients.  We make progress by crucially exploiting a form of {\em approximate} duality between sandwiching polynomials and moment matching. 

\paragraph{Moment matching and sandwiching polynomials.} The duality between fooling using moment matching and the existence of sandwiching polynomials is well-known in the setting of the Boolean hypercube \cite{bazzi2009polylogarithmic}, where moment matching $\unif\cube{d}$ up to degree $k$ is equivalent to $k$-wise independence. We need an approximate version of this duality in order to derive a bound on the coefficients of the approximating polynomials. Moreover, we need duality to hold over continuous domains for our applications to non-discrete settings (such as Gaussian and strongly logconcave distributions). A duality relating \emph{exact} moment matching and sandwiching approximations over general domains was proved in~\cite{klivans2013moment}.

We derive the following general duality result, which tells us that approximate moment matching fools a class $\calC$ over $\Dx$ iff every concept in $\calC$ admits a pair of sandwiching polynomials with sufficiently small coefficients. Our proof relies on establishing a strong duality result for a certain \emph{semi-infinite} linear program using tools from general conic duality~\cite{shapiro2001duality}.

\begin{theorem}[Fooling using approximate moment matching $\iff$ sandwiching approximation; see \cref{thm:duality}]\label{thm:duality-intro}
Let $\Dx$ be a distribution on $\calX$, and let $\calC$ be a concept class mapping $\calX$ to $\cube{}$. Let $k \in \N, \delta \in \R_+$ be degree and slack parameters, and let $\eps > 0$ be the error parameter. The following are equivalent: \begin{itemize}
    \item (Approximate moment matching fools $\calC$.) For all $f \in \calC$ and for all distributions $D'$ whose moments of degree at most $k$ are within $\delta$ of those of $\Dx$, we have $|\ex_{D'}[f] - \ex_{\Dx}[f]| \leq \eps$.
    \item (Existence of sandwiching polynomials with bounded coefficients for $\calC$.) For all $f \in \calC$, there exist degree-$k$ polynomials $p_l, p_u$ such that $p_l \leq f \leq p_u$ (pointwise over $\R^d$), and \[ \ex_{\Dx}[p_u - f] + \delta\|p_u\|_1 \leq \eps, \qquad \ex_{\Dx}[f - p_l] + \delta\|p_l\|_1 \leq \eps, \] where $\|p_u\|_1$ (resp.\ $\|p_l\|_1$) refers to the $\ell_1$ norm of the coefficients of $p_u$ (resp.\ $p_l$).
\end{itemize}
\end{theorem}

\paragraph{Applications: testably learning functions of halfspaces and more.} Combining \cref{thm:main-algorithm-intro,thm:duality-intro}, we obtain a clean framework for testable learning that reduces the task to establishing that approximate low-degree moment matching fools the target concept class over the target marginal. As our main application, we show that any function of a constant number of halfspaces over $\R^d$ can be testably learned up to excess error $\eps$ in sample and time complexity $d^{\wt{O}(1/\eps^2)}$ with respect to any distribution whose directional projections are sufficiently anticoncentrated and have strictly sub-exponentially decaying tails:

\begin{definition}\label{def:strictly-subexp-intro}
We say a distribution $\Dx$ on $\R^d$ is anticoncentrated and has $\alpha$-strictly subexponential tails if the following hold: \begin{enumerate}[(a)]
\item \emph{$\alpha$-strictly subexponential tails:}  For all $\|u\| = 1$, $\pr[|\inn{x,u}| > t] \leq \exp(-C t^{1+\alpha})$ for some constant $C$.
\item \emph{Anticoncentration:} For all $\|u\| = 1$ and continuous intervals $T \subset \R$, we have $\pr[\inn{x, u} \in T] \leq C' |T|$ for some constant $C'$.
\end{enumerate}
\end{definition}


\begin{theorem}[Testably learning functions of halfspaces; see \cref{thm:learn-hs}]\label{thm:learn-hs-intro}
Let $\calC$ be the class of functions of a constant number of halfspaces over $\R^d$. Let $\Dx$ be a distribution that is anticoncentrated and has $\alpha$-strictly subexponential tails (\cref{def:strictly-subexp-intro}). Then $\calC$ can be testably learned w.r.t.\ $\Dx$ up to excess error $\eps$ using sample and time complexity $d^{\wt{O}(\epsilon^{-(1+\alpha)/\alpha})}$.
\end{theorem}

We note that even for ordinary agnostic learning, the above result is an exponential improvement in the dependence on $\eps$ in the degree of the sandwiching polynomial compared to prior constructions of \cite{klivans2013moment}.

On the flip side, note that even though our framework handles $\Dx$ that come from a fairly broad family, our tester must know the low-degree moments of the particular $\Dx$ with respect to which it is expected to succeed. This is true for the approach of \cite{rubinfeld2022testing} as well, and it is an interesting open question whether this can be relaxed.

The class of distributions that are anticoncentrated and have strictly subexponential tails is fairly general and includes Gaussians, the uniform distribution on the unit sphere, and more generally, any strongly logconcave distribution \cite{saumard2014log} (and in fact all of these examples have $\alpha = 1$). This latter class includes the uniform distribution on any convex body with smooth boundary~\cite{MR889476} and in particular, additive Gaussian smoothening of any convex body. \cref{thm:learn-hs-intro} already matches the upper bound of \cite{kalai2008agnostically} as well as known statistical-query (SQ) lower bounds \cite{goel2020statistical,diakonikolas2020near,diakonikolas2021optimality} for ordinary agnostic learning of a single halfspace with respect to the Gaussian distribution. It also generalizes and improves the main algorithmic result of \cite{rubinfeld2022testing}, who showed such a result for a single halfspace with time and sample complexity $d^{\wt{O}(1/\eps^4)}$.

The key technical result underlying Theorem~\ref{thm:learn-hs-intro} is a proof that any distribution that approximately matches degree-$\wt{O}(\epsilon^{-(1+\alpha)/\alpha})$ moments with a distribution $\Dx$ which is anticoncentrated and has $\alpha$-strictly subexponential tails fools functions of halfspaces with respect to $\Dx$ (see \cref{thm:fool-gaussian}). Similar to the approach of \cite{klivans2013moment}, we rely on powerful methods arising from the classical theory of moments \cite{klebanov1996proximity} and metric distances in probability \cite{zolotarev1984probability, rachev2013methods} to show that whenever the moments of $\Dx$ are strictly sub-exponential, moment closeness implies distribution closeness in the so-called $\lambda$-metric (see \cref{subsec:lambda-dist}).  

We also apply our framework to immediately obtain testable learning results with respect to the $\unif\cube{d}$ in time $d^{O(k)}$ for classes $\calC$ that are fooled by $k$-wise independence, including halfspaces \cite{diakonikolas2010bounded}, degree-2 PTFs \cite{diakonikolas2010ptf}, and constant-depth circuits \cite{braverman2010polylogarithmic}, with running time and sample complexity that matches their ordinary agnostic counterparts; see \cref{thm:learn-kwise}. Over the hypercube, the fact that approximate moment matching --- i.e.\ almost $k$-wise independence --- suffices to fool such classes is immediate by a result due to \cite{alon2003almost}. 

\paragraph{Moments vs anticoncentration.} Theorem~\ref{thm:learn-hs-intro} immediately implies that one can test anticoncentration properties of all directional marginals of a broad family of distributions by checking only the low-degree moments. 

\begin{corollary}[Anticoncentration from approximate moment matching; see \cref{cor:anticonc}] \label{cor:anticonc-intro}
Fix $\eps>0$ and a distribution $\Dx$ that is anticoncentrated and has $\alpha$-strictly subexponential tails. Let $D'$ be any distribution whose moments of degree at most $k = \wt{O}(\eps^{-(1+\alpha)/\alpha})$ match those of $\Dx$ up to an additive slack of $d^{-\wt{O}(k)}$. Then for any $\|u\| = 1$ and any continuous interval $T \subset \R$, $\pr_{x \sim D'}[\inn{x, u} \in T] \leq \pr_{x \sim \Dx}[\inn{x, u} \in T] + \eps$.
\end{corollary}

This statement relates anticoncentration phenomena to structure in low-degree moments. In particular, any distribution that matches the first degree-$\wt{O}(1/\epsilon^2)$ moments of a strongly logconcave distribution must have all its directional marginals anticoncentrated up to an additive error of $\epsilon$. In addition to being a basic result in probability, such a connection relates to verifying anticoncentration of all directional marginals from a small sample. Finding verification subroutines that extend beyond Gaussian (and the uniform distribution on the sphere) have a host of applications in algorithmic robust statistics and immediately yield  efficient robust algorithms for list-decodable linear regression~\cite{DBLP:conf/nips/KarmalkarKK19,DBLP:conf/soda/RaghavendraY20} and covariance estimation~\cite{DBLP:conf/soda/BakshiK21,raghavendra2020list,DBLP:conf/stoc/IvkovK22}, and robust clustering~\cite{BK20b,DBLP:conf/focs/BakshiDHKKK20} of mixtures for broader families of distributions than currently known. 

For the specific case of Gaussian distributions (and the uniform distribution on the unit $d$-dimensional sphere), such a property for the case when $T$ is an origin centered interval was first proved in a sequence of works that introduced \emph{certifiable anticoncentration} in the context of algorithmic robust statistics~\cite{DBLP:conf/nips/KarmalkarKK19,DBLP:conf/soda/RaghavendraY20}. Their proofs use a polynomial approximator for the ``box function'' (see e.g.\ \cite[Appendix A]{DBLP:conf/nips/KarmalkarKK19}) and show that degree-$\wt{O}(1/\eps^2)$ moments are enough to ensure $\epsilon$-approximate anticoncentration for origin centered intervals $T$. A similar argument based on approximations for the box function was used by \cite{rubinfeld2022testing} to show that matching degree-$\wt{O}(1/\epsilon^4)$ moments of Gaussian implies $\epsilon$-approximate anticoncentration for all intervals $T$ as above. This quartic dependence in the order of moments required appears necessary in a proof that constructs polynomial approximations for the box function. Our argument above circumvents this bottlneck in these previous techniques and recovers the $\epsilon$-additive error anticoncentration from matching just the degree-$\wt{O}(1/\eps^2)$ moments.


\paragraph{Comparison to the algorithmic technique of \cite{rubinfeld2022testing}.} As their main algorithmic result, Rubinfeld and Vasilyan~\cite{rubinfeld2022testing} gave a testable learning algorithm for halfspaces that uses $d^{\tilde{O}(1/\epsilon^4)}$ time and samples. Their algorithm uses the fact that halfspaces admit a low-degree polynomial approximator with respect to a distribution $\Dx$ whenever $\Dx$ is anticoncentrated and has subgaussian low-degree moments. In order to verify that the empirical distribution on a large enough Gaussian sample possesses these two properties, they relate anticoncentration to low-degree moments via polynomial approximators for the box function as described above. Such a technique is possible for the simple setting of halfspaces on Gaussian distributions (with a suboptimal quartic dependence on $1/\epsilon$), but it is already unclear how to extend it to halfspaces on non-product distributions (where Fourier methods fail) or to more expressive concept classes such as functions of halfspaces. 

In contrast, an appeal to sandwiching approximation allows us to extend our testable learning results to non-anticoncentrated discrete distributions such as the uniform distribution on the hypercube, more expressive concept classes such as constant depth circuits on the hypercube and functions of halfspaces on continuous distributions, and to a broad family of distributions including all strongly logconcave distributions.



\paragraph{Sample complexity of testable learning and Rademacher complexity.} One of our main contributions is a complete characterization of the sample complexity of testable learning.  Similar to how VC-dimension corresponds to the sample complexity of distribution-free agnostic learning, we show that Rademacher complexity is the key quantity that controls the sample complexity of testable learning.  Recall that the Rademacher complexity of a class $\calC$ w.r.t.\ $\Dx$ at sample size $m$ is given by \[ \calR_m(\calC, D_\calX) = \ex_{\{x_i\}_{i \in [m]} \sim \Dx^{\otimes m}}\ \ex_{\sigma \sim \{\pm 1\}^{\otimes m}}\ \sup_{f \in \calC} \Big| \frac{1}{m} \sum_{i \in [m]} \sigma_i f(x_i) \Big|. \] This measure plays an important role in statistical learning theory since it controls the uniform convergence of empirical losses to true losses over all $f \in \calC$ (see \cref{thm:rad-unif-conv}).  We obtain precise upper and lower bounds on the sample complexity of testable learning within excess error $\epsilon$ purely in terms of Rademacher complexity:


\begin{theorem}[Rademacher Complexity Characterizes Testable Learning, see \cref{thm:upper-bound,,thm:lower-bound}]\label{thm:samp-comp-intro}
Let $\Dx$ be a distribution on $\calX$, let $\calC$ be a concept class mapping $\calX$ to $\cube{}$, and let $\eps > 0$ be the error parameter. \begin{itemize}
    \item (Upper bound.) Let $m$ be such that $\calR_m(\calC, \Dx) \leq \eps/5$. Then $\calC$ can be testably learned  w.r.t.\ $\Dx$ up to excess error $\eps$ using sample complexity $m + O(1/\eps^2)$.
    \item (Lower bound.) Let $M$ be such that $\calR_M(\calC, \Dx) \geq 5\eps$, and assume $M \geq \Theta(1/\eps^2)$. Then the sample complexity required to testably learn $\calC$ w.r.t.\ $\Dx$ up to excess error $\eps$ is at least $\Omega(\sqrt{M})$.
\end{itemize}
\end{theorem}

This characterization yields an interesting separation between ordinary distribution-specific agnostic learning and testable learning. For the former, while uniform convergence is always a sufficient condition, it is not necessary, as witnessed by examples such as convex sets in Gaussian space \cite{klivans2008learning} and monotone Boolean functions \cite{bshouty1996fourier} (see \cref{subsec:apps}). Indeed, the sample complexity of distribution-specific agnostic learning is known to be characterized by the metric entropy rather than the Rademacher complexity \cite{benedek1991learnability}. In contrast, we see that the Rademacher complexity provides the right characterization of testable learning. Thus, testable learning is a natural supervised learning model for which bounded Rademacher complexity, and hence uniform convergence, provides a \emph{necessary and sufficient} condition for learning. For further discussion, see \cref{subsec:unif-cov}.

\subsection{Concurrent work}
In independent and concurrent work, Rubinfeld and Vasilyan \cite{rubinfeld2022personal} have extended their algorithm for halfspaces with respect to Gaussian target marginals to the uniform distribution over the hypercube. They do so by reusing their approximator for the box function and showing that it yields a polynomial approximator for \emph{regular} halfspaces with respect to almost $k$-wise independent distributions. They then utilize the ``critical index'' framework of \cite{diakonikolas2010bounded} to reduce the case of general halfspaces to the regular case. Their tester and its analysis rely on $\ell_1$-approximating polynomials for (regular) halfpsaces (instead of sandwiching approximations as in our work) and incurs a suboptimal $d^{\tilde{O}(1/\epsilon^4)}$ time and sample complexity as opposed to the (conjecturally) optimal $d^{\tilde{O}(1/\epsilon^2)}$ bound obtained by our approach.

\subsection{Related work}

The duality between fooling using bounded independence and sandwiching approximation is a fundamental tool in the pseudorandomness literature for showing that $k$-wise independence fools various classes \cite{bazzi2009polylogarithmic,braverman2010polylogarithmic,diakonikolas2010bounded}. Its more general statement in terms of moment matching was observed by \cite{klivans2013moment} (see also \cite{kane2013learning}), who used it to obtain low-degree sandwiching polynomials for functions of halfspaces w.r.t.\ logconcave distributions (their constructions do not give any insight on the $\ell_1$ norm of the coefficients). We build on their approach for our main application, namely testably learning functions of halfspaces with respect to Gaussians, and obtain exponentially improved degree bounds in terms of $\epsilon$ along with effective bounds on the size of the coefficients.

In statistical learning theory and nonparametric regression, one of the basic objectives is to place tight bounds on the excess risk $L(\hat{f}) - \inf_{f \in \calC} L(f)$ and on the generalization gap $|\hat{L}_m(\hat{f}) - L(\hat{f})|$ of an estimator $\hat{f}$ in various settings (see e.g.\ \cite{birge1997model,van2000empirical,tsybakov2008introduction}). In particular, there is a long line of work studying data-dependent bounds on these quantities in terms of measures such as the Rademacher complexity and various refinements and variants thereof \cite{koltchinskii2000rademacher,koltchinskii2001rademacher, bartlett2002model,bartlett2002rademacher,bartlett2005local,koltchinskii2006local}. Our sample complexity upper bound applies a simple such data-dependent bound to the testable learning setting. In terms of lower bounds, \cite{kur2021minimal} have studied bounds on the minimal error of any ERM estimator, in the additive Gaussian noise setting, in terms of the Gaussian complexity. None of these works, however, consider a model similar to testable learning.

Statistical characterizations of PAC learning have also been well-studied. In the distribution-free setting, it is very well-known that the sample complexity is characterized fully by the VC-dimension, and equivalent to uniform convergence \cite{vapnik1998statistical}. For distribution-specific agnostic setting, \cite{benedek1991learnability} obtained a characterization in terms of the metric entropy, or the log covering number. Work by \cite{kothari2018improper} (see also \cite{vadhan2017learning}) proposed a characterization of efficient agnostic learning using the so-called refutation complexity, and interpreted it as a computational analog of Rademacher complexity. Results of \cite{shalev2010learnability} showed that in Vapnik's General Setting of Learning, the sample complexity is in general characterized by notions of algorithmic stability rather than uniform convergence. In modern deep learning theory, the failures of the uniform convergence paradigm in the overparameterized regime have been much studied (see e.g.\ \cite{zhang2021understanding,nagarajan2019uniform}); we refer the reader to \cite{bartlett2021deep,belkin2021fit} for surveys.

\section{Preliminaries}

\subsection{Notation and conventions}
We denote the domain by $\calX$, which for us is always either $\R^d$ or $\cube{d}$, and labels always lie in $\cube{}$. We use $\calC$ to denote a concept class mapping $\calX$ to $\cube{}$. We use $\Dx$ to denote a well-behaved distribution on $\calX$ (i.e.\ the target marginal, such as $\calN(0, I_d)$ or $\unif\cube{d}$), and we use the calligraphic $\calD$ to denote labeled distributions on $\calX \times \cube{}$. We denote a size-$m$ (labeled) sample drawn from $\calD$ by $S \sim \calD^{\otimes m}$. If $S = \{(x_i, y_i)\}_{i \in [m]}$, then we use $S_\calX = \{x_i\}_{i \in [m]}$ to denote its ``marginal'', i.e.\ the unlabeled sample.

Our loss function throughout will be the 0-1 loss function, $\ell(\hat{y}, y) = \ind[\hat{y} \neq y]$. Given a labeled distribution $\calD$, we denote the population loss by $L(f, \calD) = \pr_{(x,y) \sim \calD}[f(x) \neq y]$ (or just $L(f)$ when $\calD$ is implicit), and the empirical loss on a size-$m$ sample $S \sim \calD^{\otimes m}$ by $\hat{L}_m(f, S) = \pr_{(x_i, y_i) \sim S}[f(x_i) \neq y_i]$ (or just $\hat{L}_m(f)$ when $S$ is implicit). We follow the convention of denoting empirical quantities using a hat and a subscript to denote the sample size (as in $\hat{L}_m$). We use $\opt(\calC, \calD)$ to denote $\inf_{f \in \calC}L(f, \calD)$.

We follow the following conventions when working with monomials over $\calX$. For any multi-index $I \in \N^d$, let $|I| = \sum_j I_j$ denote its degree (or sometimes order), and let $x_I$ denote the monomial $\prod_{j \in [d]} x_j^{I_j}$. We use $\calI(k,d) = \{I \in \mathbb{N}^d \mid |I| \leq k \}$ to denote the set of multi-indices of degree at most $k$. For a vector $\Delta \in \R_+^{|\calI(k,d)|}$ and a degree-$k$ polynomial $p : \R^d \to \R$ given by $p(x) = \sum_{I \in \calI(k,d)} p_I x_I$, we use $\inn{\Delta, |p|}$ to denote $\sum_{I \in \calI(k,d)} |p_I| \Delta_I$. This may be thought of as the $\Delta$-weighted $\ell_1$ norm of the coefficients of $p$.

We use $a \eqsim b$ and $a \lesssim b$ to denote equalities and inequalities up to constants. It will be convenient to state Stirling's approximation in the following form \cite{robbins1955remark}: for all $n \geq 1$, $n! \eqsim e^{-n} n^{n + 1/2}$. We will also make use of the double factorial $n!! = n(n-2) \cdots$, satisfying $n! = n!! (n-1)!!$ for even $n$.

Throughout this paper, we use the term ``with high probability'' to mean ``with probability at least $0.99$'' (or any other sufficiently large constant) for simplicity. In all cases, confidence may be amplified using standard repetition arguments.

\subsection{Learning models}

Here we formally define the learning models we work with. Let $\Dx$ be a distribution on $\calX$, where $\calX$ is either $\R^d$ or $\cube{d}$, and let $\calC$ be a concept class mapping $\calX$ to $\cube{}$.

\begin{definition}[Distribution-specific agnostic learning]
We say a learner $A$ agnostically learns $\calC$ w.r.t.\ $\Dx$ up to excess error $\eps$ if for any $\calD$ on $\calX \times \cube{}$ with marginal $\Dx$, given sufficiently many examples drawn from $\calD$, with high probability $A$ outputs a hypothesis $h$ such that $L(h) \leq \opt(\calC, \calD) + \eps$. Here, recall that $L(f) = \pr_{(x,y) \sim \calD}[f(x) \neq y]$ and $\opt(\calC, \calD) = \inf_{f \in \calC} L(f)$.
\end{definition}

We recall the standard result of \cite{kalai2008agnostically} that shows that polynomial approximators with respect to the absolute loss yield agnostic learning algorithms.
\begin{theorem}[\cite{kalai2008agnostically}]\label{thm:kkms}
Suppose that for every $f \in \calC$, there exists a degree-$k$ polynomial $p : \calX \to \R$ such that $\ex_{x \sim \Dx}[|f(x) - p(x)|] \leq \eps$. Then there exists a simple agnostic learner (based on degree-$k$ polynomial regression w.r.t.\ the absolute loss) for learning $\calC$ w.r.t.\ $\Dx$ up to excess error $\eps$ in time and sample complexity $d^{O(k)}/\poly(\eps)$.
\end{theorem}

We now formally define testable learning.
\begin{definition}[Testable agnostic learning, \cite{rubinfeld2022testing}]\label{def:testable-learning}
We say a tester-learner pair $(T, A)$ testably learns $\calC$ w.r.t.\ $\Dx$ up to excess error $\eps$ if for any distribution $\calD$ on $\calX \times \cube{}$, the following conditions are met: \begin{itemize}
    \item (Soundness/composability.) If $\calD$ is such that the tester $T$ accepts with high probability over a sample drawn from $\calD$, then the learner $A$ succeeds in agnostically learning $\calC$ w.r.t.\ $\calD$, i.e.\ with high probability it produces a hypothesis $h$ such that $L(h) \leq \opt(\calC, \calD) + \eps$.
    \item (Completeness.) Whenever $\calD$ truly has marginal $\Dx$ on $\calX$, the tester $T$ accepts with high probability.
\end{itemize}
Again, here ``with high probability'' may be taken to be ``with probability at least $0.99$'' for simplicity, and the confidence in each step may be amplified using standard repetition arguments.
\end{definition}

Note that as stated, it is not strictly necessary for the tester $T$ and the learner $A$ to work with the same sample, and the definition may be interpreted as saying ``if $T$ accepts $\calD$ (w.h.p.), then $A$ must succeed over $\calD$ (w.h.p.)''. The algorithms and characterizations we give in this paper have a stronger ``data-dependent'' guarantee, where both $T$ and $A$ operate on the same sample $S$ drawn from $\calD$, and have the following interpretation: ``if $T$ accepts $S$, then $A$ must succeed over $S$ (as well as generalize to $\calD$ w.h.p.)''.

\subsection{Bounded independence and sandwiching polynomials over the hypercube}
In this section, let $U$ denote $\unif \cube{d}$.

\begin{definition}
We say a distribution $D$ on $\cube{d}$ is $(\delta, k)$-independent if for all $|I| \leq k$, $|\ex_{D}[x_I]| \leq \delta$. When $\delta = 0$, we simply call $D$ a $k$-wise independent distribution. 

We say that a concept class $\calC$ is $\eps$-fooled by $(\delta, k)$-independence (resp.\ $k$-wise independence) if for every $f \in \calC$ and any $(\delta, k)$-independent (resp.\ $k$-wise independent) $D$, $|\ex_D[f] - \ex_{U}[f]| \leq \eps$.
\end{definition}

Notice that saying $D$ is $(\delta, k)$-independent is exactly equivalent to saying that the moments of degree at most $k$ of $D$ are within $\delta$ of those of $U$ (which, of course, are all $0$).

We now recall a fundamental duality result in pseudorandomness, which states that bounded independence fools a class $\calC$ iff it admits sandwiching polynomials w.r.t.\ $U$.

\begin{theorem}[{\cite[Thm A.1]{bazzi2009polylogarithmic}}]\label{thm:duality-hypercube}
Let $f : \cube{d} \to \cube{}$, let $\eps > 0$ be the error parameter, and let $k \in \N, \delta > 0$ be the degree and slack parameters. The following are equivalent: \begin{enumerate}[(a)]
    \item $(\delta, k)$-independence fools $f$.
    \item There exist degree-$k$ polynomials $p_l, p_u$ such that $p_l \leq f \leq p_u$ (pointwise over $\cube{d}$), and \[ \ex_{U}[p_u - f] + \delta \|p_u\|_1 \leq \eps, \qquad \ex_{U}[f - p_l] + \delta \|p_l\|_1 \leq \eps, \] where for a polynomial $p(x) = \sum_I p_I x_I$ we use $\|p\|_1$ to denote $\sum_{I \neq 0} |p_I|$, i.e.\ the $\ell_1$ norm of its (nonconstant) coefficients.
\end{enumerate}
\end{theorem}

The following theorem, showing that a $(\delta, k)$-independent distribution is statistically close to being $k$-wise independent, will also be useful to us.

\begin{theorem}[{\cite[Thm 2.1]{alon2003almost}}]\label{thm:almost-kwise}
Let $D$ be a $(\delta, k)$-independent distribution on $\cube{d}$. Then there exists a $k$-wise independent distribution $D'$ that has TV distance at most $\delta d^k$ from $D$.
\end{theorem}

\subsection{Rademacher complexity}

The Rademacher complexity is one of the most well-studied measures of the complexity of a function class in statistical learning theory, and may be intuitively thought of as measuring the ability of a function class to fit a randomly-labeled sample. The following definitions and theorems are now standard in the literature (see e.g.~\cite{bousquet2003introduction,bartlett2002rademacher,bartlett2014notes,bartlett2021deep} and references therein).

\begin{definition}[Rademacher complexity]
Consider a sample of $m$ points $S_\calX = \{x_i\}_{i \in [m]} \sim \Dx^{\otimes m}$. The empirical Rademacher complexity of the class $\calC$ w.r.t.\ this sample is defined to be \begin{equation}
    \hat{\calR}_m(\calC, S_\calX) = \ex_{\sigma \sim \{\pm 1\}^{\otimes m}} \sup_{f \in \calC} \Big| \frac{1}{m} \sum_{i \in [m]} \sigma_i f(x_i) \Big|.
\end{equation} Note that this is a random variable depending on $S_\calX$. The (expected) Rademacher complexity of $\calC$ w.r.t.\ $\Dx$ at sample size $m$ is defined to be \begin{equation}
    \calR_m(\calC, \Dx) = \ex_{S_{\calX} \sim \Dx^{\otimes m}} \hat{\calR}_m(\calC, S_{\calX}).
\end{equation} Sometimes we simply say $\calR_m(\calC)$ (resp.\ $\hat{\calR}_m(\calC)$) when $\Dx$ (resp.\ $S_\calX$) is clear from context.
\end{definition}

The next theorem states that the Rademacher complexity of a class tightly controls uniform convergence, i.e.\ bounds on the quantity $\sup_{f \in \calC}|L(f) - \hat{L}_m(f)|$, where $L$ and $\hat{L}_m$ are the population and empirical loss functionals. The upper bound here follows from a so-called symmetrization argument, while the lower bound follows from a desymmetrization argument. In our statement, we specialize to the case of the 0-1 loss.\footnote{For general loss functions $\ell$, one would define a ``loss class'' $\ell \circ \calC = \{ (x,y) \mapsto \ell(f(x), y) \mid f \in \calC \}$ and state such a result in terms of $\calR_m(\ell \circ \calC)$. In the case of the 0-1 loss function, it is known that $\calR_m(\ell \circ \calC) = \frac{1}{2}\calR_m(\calC)$.}

\begin{theorem}[see e.g.\ \cite{bartlett2014notes}]\label{thm:rad-unif-conv}
Let $\calC$ be a class of functions mapping $\calX$ to $\cube{}$. Let $\calD$ be a distribution on $\calX \times \cube{}$ with marginal $\Dx$ on $\calX$, and let $S \sim \calD^{\otimes m}$ be a random sample of size $m$ drawn from $\calD$. For any $f \in \calC$, let $L(f) = \pr_{(x,y) \sim \calD} [f(x) \neq y]$, and let $\hat{L}_m(f) = \pr_{(x_i,y_i) \sim S}[f(x_i) \neq y_i]$. Then with probability $1 - \delta$ over the draw of $S$, we have \[  \frac{1}{4}\calR_m(\calC) - \Theta\left(\sqrt{\frac{\log(1/\delta)}{m}}\right) \leq \sup_{f \in \calC} \left| L(f) - \hat{L}_m(f) \right| \leq \calR_m(\calC) + \Theta\left(\sqrt{\frac{\log(1/\delta)}{m}}\right). \] 
\end{theorem}

The following useful facts characterize the concentration of the quantities defining the Rademacher complexity and follow by standard applications of McDiarmid's inequality. Assume that the range of $\calC$ is bounded in $[-1,1]$.

The first fact is that the empirical Rademacher complexity $\hat{\calR}_m(\calC)$ concentrates tightly around the expected Rademacher complexity $\calR_m(\calC)$. Formally, with probability at least $1 - \delta$ over a sample $S_\calX = \{x_i\}_{i \in [m]} \sim \Dx^{\otimes m}$ of size $m$, we have
\begin{equation}\label{eq:emp-rad-conc} \left|\calR_m(\calC, \calD_{\calX}) - \hat{\calR}_m(\calC, S_\calX) \right| \leq O\left(\sqrt{\frac{\log (1/\delta)}{m}} \right). \end{equation}

The second is that for any fixed sample $S$ of size $m$, the random variable $\sup_{f \in \calC} \frac{1}{m} \sum_i \sigma_i f(x_i)$ concentrates tightly around its expectation, $\hat{\calR}_m(\calC)$. Formally, with probability at least $1 - \delta$ over $\sigma \sim \{\pm 1\}^m$ we have
\begin{equation}\label{eq:emp-rad-conc-2} \left|\hat{\calR}_m(\calC, S_\calX) - \sup_{f \in \calC} \frac{1}{m} \sum_i \sigma_i f(x_i) \right| \leq O\left(\sqrt{\frac{\log (1/\delta)}{m}} \right). \end{equation}

Combining \cref{thm:rad-unif-conv} with \cref{eq:emp-rad-conc}, we actually have the following data-dependent generalization guarantee for a sample $S$ in terms of $\hat{\calR}_m(\calC, S_\calX)$ itself: with probability at least $1-\delta$ over the draw of $S$, for every $f \in \calC$, \begin{equation}\label{eq:rad-comp-data} \left| L(f) - \hat{L}_m(f) \right| \leq \hat{\calR}_m(\calC) + O \left(\sqrt{\frac{\log(1/\delta)}{m}} \right).
\end{equation}

\section{Duality}

In this section we state the duality between fooling using approximate moment matching and sandwiching polynomials. This is a generalization of duality over the hypercube, \cref{thm:duality-hypercube}, to continuous domains and more general distributions. Note that a version of duality over $\R^d$, albeit only for exact moment matching, was stated in \cite[Lemma 3.3]{klivans2013moment}.

\begin{definition}[Approximate moment matching]\label{def:moment-match}
Let $k \in \N$ be a degree parameter, and let $\Delta \in \mathbb{R}_+^{|\calI(k,d)|}$ be a slack parameter, satisfying $\Delta_{0} := \Delta_{(0, \dots, 0)} = 0$ and $\Delta_I > 0$ for all other $I \in \calI(k,d)$. We say that two distributions $D, D'$ on $\calX$ match moments of degree (or order) at most $k$ up to slack $\Delta$ if $|\ex_{D}[x_I] - \ex_{D'}[x_I]| \leq \Delta_I$ for all $I \in \calI(k,d)$.
\end{definition}
The reason for allowing the slack $\Delta_I$ to depend on $I$ is that in general we expect the scale of the moments to vary widely with $I$ (as with the Gaussian, for example). The empty index $I_0 = 0 = (0, \dots, 0)$ plays a special role, since $x_{I_0} = 1$ and $\ex_{D}[1] = 1$ for any valid distribution, meaning we may assume $\Delta_0 = 0$ without loss of generality.

We can now state the main theorem. We prove this theorem using conic LP duality \cite{shapiro2001duality}, taking care to establish strong duality, but the essential argument is similar to Bazzi's proof of \cref{thm:duality-hypercube}.

\begin{theorem}\label{thm:duality}
Let $k \in \N, \Delta \in \mathbb{R}_+^{|\calI(k,d)|}$ be the degree and slack parameters, as in \cref{def:moment-match}. Let $f : \calX \to \R$ be a function, and let $D$ be a distribution on $\calX$. The following are equivalent: \begin{enumerate}[(a)]
\item (Approximate moment matching fools $f$ w.r.t.\ $D$.) For any distribution $D'$ whose moments up to  order $k$ match those of $D$ up to $\Delta$, we have $|\ex_D[f] - \ex_{D'}[f]| \leq \eps$.
\item (Existence of sandwiching polynomials with bounded coefficients for $f$ w.r.t.\ $D$.) There exist degree-$k$ polynomials $p_l, p_u$ such that $p_l \leq f \leq p_u$ (pointwise over $\R^d$), and \[ \ex_{D}[p_u - f] + \inn{\Delta, |p_u|} \leq \eps, \qquad \ex_{D}[f - p_l] + \inn{\Delta, |p_l|} \leq \eps. \] (Recall that for a degree-$k$ polynomial $p(x) = \sum_{I} p_I x_I$, we use $\inn{\Delta, |p|}$ to denote $\sum_{I \in \calI(k,d)} |p_I| \Delta_I$.)
\end{enumerate}
\end{theorem}
\begin{proof}
Let $\sigma_I = \ex_{D}[x_I]$. Let $P_d$ be the set of all Borel probability measures on $\R^d$. Consider the following semi-infinite linear program, which seeks to maximize $\ex_{D'}[f]$ over all probability distributions $D'$ on $\R^d$ that approximately match moments with $D$: \begin{align}
    \sup_{D' \in P_d} &\ex_{D'}[f] \label{eq:lp-primal} \\
    \text{ subject to } \qquad \sigma_I - \Delta_I \leq &\ex_{D'}[x_I] \leq \sigma_I + \Delta_I \quad \forall I \in \calI(k,d)
\end{align}
The case of $I = (0, \dots, 0)$ is special: here $\sigma_{0} = \ex_{D}[1] = 1$ and $\Delta_{0} = 0$, so the corresponding constraint becomes simply $\ex_{D'}[1] = 1$, which is equivalent to requiring that $D'$ be a valid probability measure.

The dual LP turns out to be equivalent to the following, with variable $\beta \in \R^{|\calI(k,d)|}$:  \begin{align}
    \inf_{\beta \in \R^{|\calI(k,d)|+1}} &\sum_{I \in \calI(k,d)} \beta_I \sigma_I + \sum_{I \in \calI(k,d)} |\beta_I| \Delta_I \label{eq:lp-dual} \\
    \text{ subject to } \qquad &\sum_{I \in \calI(k,d)}\beta_I x_I \geq f(x)\quad \forall x \in \R^d
\end{align}
Notice that the primal LP (\cref{eq:lp-primal}) is feasible (indeed, by $D' = D$), and moreover, we claim that strong duality holds. Accepting this for a moment, denote the common optimum of \cref{eq:lp-primal,eq:lp-dual} by $\gamma$. The claim that approximate moment matching fools $f$ (in a one-sided fashion) w.r.t.\ $D$ is the same as asserting $\gamma \leq \ex_{D}[f] + \eps$. Take $\beta$ to be an optimal solution to the dual, and let $p_u(x) = \sum_{I \in \calI(k,d)} \beta_I x_I$. (In fact, this correspondence between degree-$k$ polynomials and their coefficient vectors allows us to equivalently view the dual as optimizing over such polynomials instead of their coefficients.) The dual then tells us that $p_u \geq f$ pointwise, and \[ \gamma = \sum_{I \in \calI(k,d)} \beta_I \sigma_I + \sum_{I \in \calI(k,d)} |\beta_I| |\Delta_I| = \ex_{D}[p_u] + \inn{\Delta, |p_u|} \leq \ex_{D}[f] + \eps, \] establishing the existence of the upper sandwiching polynomial. To obtain the lower sandwiching polynomial, we replace the objective of the primal with $-\ex_{D'}[f]$ and repeat the same argument, this time using the fact that the common optimum $\gamma'$ satisfies $\gamma' \leq -\ex_{D}[f] + \eps$ (i.e., effectively replacing $f$ with $-f$ throughout). This establishes the desired equivalence if we accept strong duality.

Formally, it remains to properly justify that the primal LP (\cref{eq:lp-primal}) is well-posed, that \cref{eq:lp-dual} is indeed the dual of \cref{eq:lp-primal}, and that strong duality holds. We do so in \cref{sec:formal-duality} using results from general conic LP duality \cite{shapiro2001duality}.
\end{proof}

To see how \cref{thm:duality-intro} may be recovered from this, simply set $\Delta$ to be $\delta$ in every coordinate.

For the purposes of testing using moment matching, one can only ever hope to check that the unknown marginal ($D'$, say) \emph{approximately} matches moments with the target marginal. Approximate duality --- and specifically the appearance of the quantities $\inn{\Delta, |p_u|}$ and $\inn{\Delta, |p_l|}$ --- turns out to be precisely what we need to guarantee sandwiching polynomials even w.r.t.\ such a $D'$.

\begin{corollary}\label{cor:sandwiching}
Let $f, D, k, \Delta, \eps$ satisfy the conditions of \cref{thm:duality}, and let $p_l \leq f \leq p_u$ be the resulting sandwiching polynomials for $f$ w.r.t.\ $D$. Consider any particular $D'$ whose moments up to order $k$ match those of $D$ up to $\Delta$. Then $p_l, p_u$ are sandwiching polynomials for $f$ w.r.t. $D'$ as well, satisfying \[ \ex_{D'}[p_u - f] \leq 2\eps, \qquad \ex_{D'}[f - p_l] \leq 2\eps. \]
\end{corollary}
\begin{proof}
By the first part of \cref{thm:duality}, we know $|\ex_{D}[f] - \ex_{D'}[f]| \leq \eps$. Thus \begin{align}
\left| \ex_{D}[p_u - f] - \ex_{D'}[p_u - f] \right| &\leq \left| \ex_{D}[p_u] - \ex_{D'}[p_u] \right| + \left| \ex_{D}[f] - \ex_{D'}[f] \right| \\
&\leq \inn{\Delta, |p_u|} + \eps.
\end{align} Applying the second part of \cref{thm:duality}, this means \begin{align}
\ex_{D'}[p_u - f] &\leq \ex_{D}[p_u - f] + \inn{\Delta, |p_u|} + \eps \\
&\leq 2\eps.
\end{align}
The argument for $p_l$ is exactly the same. 
\end{proof}

\section{Testable learning via moment matching}

\subsection{Warm-up: testable learning over the hypercube via $k$-wise independence}

The main ideas of our approach are already illustrated in the setting of the Boolean hypercube. A key technical ingredient for us will be the fact that an almost $k$-wise independent distribution is statistically close to being truly $k$-wise independent (\cref{thm:almost-kwise}).

\begin{theorem}\label{thm:learn-kwise}
Let $\calC$ be any concept class that is $\frac{\eps}{4}$-fooled by $k$-wise independence. Then $\calC$ can be testably learned  w.r.t.\ $\unif\cube{d}$ up to excess error $\eps$ with time and sample complexity $d^{O(k)}/\eps^2$.
\end{theorem}
\begin{proof}
Let the unknown labeled distribution be $\calD$. Let $S \sim \calD^{\otimes m}$ be the labeled sample given to $(T,A)$ (where the sample size $m$ will be picked later), and let $S_\calX$ be its (unlabeled) marginal. Let $\hat{D}_m$ be the induced empirical distribution, i.e.\ the uniform distribution over $S_\calX$.

The tester $T$ and algorithm $A$ are simple: the tester checks that the empirical moments (or biases) up to degree $k$ are all no larger than $\delta = \eps d^{-k}/4$ in magnitude (i.e.\ that the empirical distribution is $(\delta, k)$-independent), and the algorithm runs degree-$k$ polynomial regression (\cref{thm:kkms}) over the sample.

It is clear that when $\calD$ indeed has marginal exactly $\unif\cube{d}$ (or indeed any $(\delta/2, k)$-independent distribution), then by taking $m = d^{k}/\delta^2 = d^{\Theta(k)}/\eps^2$ sufficiently large, we can ensure with high probability all the empirical moments of order at most $k$ concentrate about their true moments up to $\delta$ (by a standard Hoeffding plus union bound). That is, with high probability $\hat{D}_m$ will indeed be $(\delta, k)$-independent, and the tester will accept. This verifies completeness.

To verify soundness, suppose that $\hat{D}_m$ is indeed $(\delta, k)$-independent. By \cref{thm:almost-kwise}, this means that $\hat{D}_m$ has TV distance at most $\delta d^k = \eps/4$ from a truly $k$-wise independent distribution. This in turn means that $\hat{D}_m$ (and indeed any $(\delta, k)$-independent distribution) $\frac{\eps}{2}$-fools $\calC$. We now appeal to duality, stated here in some generality as \cref{thm:duality}, although in the setting of the hypercube this theorem reduces exactly to the form in \cref{thm:duality-hypercube}. Formally, observe that for every $f \in \calC$, condition (a) of \cref{thm:duality} is satisfied (with $D = \unif\cube{d}$, and where the slack parameter $\Delta$ is now simply $\delta$ in every coordinate). This allows us to apply \cref{cor:sandwiching} to conclude that there exist $\eps$-sandwiching polynomials for $\calC$ w.r.t.\ $\hat{D}_m$. By \cref{thm:kkms}, this ensures the learner succeeds at learning $\calC$ up to error $\opt(\calC, \calD) + \eps$ with high probability. This proves the theorem.
\end{proof}

We may apply this theorem to obtain testable learning w.r.t.\ $\unif\cube{d}$ for halfspaces, degree-2 PTFs, and constant-depth circuits.

\begin{corollary}
Let $\calC$ be the class of halfspaces over $\cube{d}$. Let $\eps > 0$, and let $k = \wt{O}(1/\eps^2)$. Then $\calC$ is $\eps$-fooled by $k$-wise independence \cite{diakonikolas2010bounded}, and hence it can be testably learned w.r.t. $\unif\cube{d}$ up to excess error $\eps$ with time and sample complexity $d^{O(k)}$.
\end{corollary}

\begin{corollary}
Let $\calC$ be the class of degree-$2$ polynomial threshold functions over $\cube{d}$. Let $\eps > 0$, and let $k = \wt{O}(1/\eps^9)$. Then $\calC$ is $\eps$-fooled by $k$-wise independence \cite{diakonikolas2010ptf}, and hence it can be testably learned w.r.t. $\unif\cube{d}$ up to excess error $\eps$ with time and sample complexity $d^{O(k)}$.
\end{corollary}

\begin{corollary}
Let $\calC$ be the class of depth-$t$ $\acz$ circuits of size $s$ over $\cube{d}$. Let $\eps > 0$, and let $k = (\log s)^{O(t)} \log (1/\eps)$. Then $\calC$ is $\eps$-fooled by $k$-wise independence \cite{braverman2010polylogarithmic,tal2017tight,harsha2019polynomial}, and hence it can be testably learned w.r.t. $\unif\cube{d}$ up to excess error $\eps$ with time and sample complexity $d^{O(k)}$.
\end{corollary}

\subsection{A general algorithm using moment matching}

We now give a more general algorithm for testable learning that does not need the target distribution to be $k$-wise independent. In this case, our tester will check that the low-degree moments of the empirical distribution are close to those of the target distribution. The correctness of our tester is a consequence of duality (\cref{thm:duality}).

\begin{theorem}\label{thm:main-algorithm}
Let $\Dx$ be a distribution on $\calX$, and let $\calC$ be a concept class mapping $\calX$ to $\cube{}$. Let $k \in \N, \Delta \in \mathbb{R}_+^{|\calI(k,d)|}$ be the degree and slack parameters, as in \cref{def:moment-match}, and let $\eps > 0$ be the error parameter. Suppose the following conditions hold: \begin{enumerate}[(a)]
\item (Empirical moments concentrate around true moments.)  There exists $m$ large enough that with high probability over a sample $S_\calX \sim \Dx^{\otimes m}$, the corresponding empirical distribution $\hat{D}_m$ matches moments of degree at most $k$ with $\Dx$ up to slack $\Delta$.
\item (Existence of sandwiching polynomials with bounded coefficients for $\calC$, or equivalently approximate moment matching fools $\calC$.) For every $f \in \calC$, there exist degree-$k$ sandwiching polynomials $p_l \leq f \leq p_u$ such that \[ \ex_{\Dx}[p_u - f] + \inn{\Delta, |p_u|} \leq \frac{\eps}{2}, \qquad \ex_{\Dx}[f - p_l] + \inn{\Delta, |p_l|} \leq \frac{\eps}{2}. \] (Recall that for a degree-$k$ polynomial $p(x) = \sum_{I} p_I x_I$, we use $\inn{\Delta, |p|}$ to denote $\sum_{I \in \calI(k,d)} |p_I| \Delta_I$.)

Equivalently, for every $f \in \calC$ and for any distribution $D'$ whose moments up to  order $k$ match those of $\Dx$ up to $\Delta$, we have $|\ex_D[f] - \ex_{D'}[f]| \leq \frac{\eps}{2}$.
\end{enumerate}

Then $\calC$ can be testably learned w.r.t.\ $\Dx$ up to excess error $\eps$ using time and sample complexity $m + d^{O(k)}$. Moreover, the tester $T$ and learner $A$ are simple: $T$ tests whether the empirical moments up to order $k$ match those of $\Dx$ up to $\Delta$, and $A$ performs degree-$k$ polynomial regression over the sample (\cref{thm:kkms}).
\end{theorem}
\begin{proof}
Let the unknown labeled distribution be $\calD$, and let $S \sim \calD^{\otimes m}$ be the sample given to $(T,A)$. First we verify completeness. By assumption, when $\calD$ indeed has marginal $\Dx$, then $m$ is large enough that with high probability over $S$, the empirical moments concentrate about the true moments up to $\Delta$, and hence $T$ accepts.

As for soundness, suppose that $T$ accepts, i.e.\ that the empirical distribution $\hat{D}_m$ indeed matches order-$k$ moments with $\Dx$ up to $\Delta$. Observe that our condition (b) is the same as condition (b) of \cref{thm:duality} is satisfied. Thus we may apply \cref{cor:sandwiching} (with $D = \Dx$ and $D' = \hat{D}_m$) to conclude that there exist degree-$k$ $\eps$-sandwiching polynomials for $\calC$ w.r.t.\ $\hat{D}_m$. By \cref{thm:kkms}, we have that degree-$k$ polynomial regression achieves error $\opt(\calC, \calD) + \eps$ with high probability. (This implicitly assumes that the degree-$k$ polynomial fitting $S$ will generalize to $\calD$, which will be true by classic VC theory whenever $m \geq d^{O(k)}/\eps^2$ since the VC dimension of degree-$k$ polynomials (with bounded coefficients, as here) is at most $d^{O(k)}$. If this is not the case, we may replace $m$ with $m + d^{O(k)}/\eps^2$.)
\end{proof}

To see how \cref{thm:main-algorithm-intro} may be recovered from this, simply set $\Delta$ to be $\delta$ in every coordinate, and also rescale $\eps$ appropriately.

\section{Testably learning functions of halfspaces over strictly subexponential distributions}

In this section we apply \cref{thm:main-algorithm} to prove that we can testably learn functions of halfspaces w.r.t.\ a target marginal $\Dx$ on $\calX = \R^d$ that is anticoncentrated and has strictly subexponential tails in the sense of \cref{def:strictly-subexp-intro} from the introduction.

\begin{definition}[Restatement of \cref{def:strictly-subexp-intro}]\label{def:strictly-subexp}
We say a distribution $\Dx$ on $\R^d$ is anticoncentrated and has $\alpha$-strictly subexponential tails if the following hold: \begin{enumerate}[(a)]
\item For all $\|u\|_2 = 1$, $\pr[|\inn{x,u}| > t] \leq \exp(-C_1 t^{1+\alpha})$ for some constant $C_1$.
\item For all $\|u\|_2 = 1$ and $k \in \N$, $\ex[|\inn{x,u}|^k]^{1/k} \leq C_2 k^{1/(1 + \alpha)}$ for some constant $C_2$.
\item For all $\|u\|_2 = 1$ and continuous intervals $T \subset \R$, we have $\pr[\inn{x, u} \in T] \leq C_3 |T|$ for some constant $C_3$.
\end{enumerate}
\end{definition}
The first two conditions are a strengthening of the usual definition of subexponential distributions (see e.g.\ \cite{vershynin2018high}), and standard arguments show that the two are actually equivalent. The third asks directional marginals of $\Dx$ to be anticoncentrated. Examples include all \emph{strongly} logconcave distributions, which satisfy this definition with $\alpha = 1$ (see e.g.\ \cite[\S5.1]{saumard2014log} or \cite[Thm 2.15]{ledoux2001concentration}). This class includes the standard Gaussian distribution, the uniform distribution on the unit $d$-dimensional sphere and more generally, uniform distribution on any convex body with smooth boundaries (e.g., Gaussian smoothening of arbitrary convex bodies).

Throughout this section, let $\calC$ be the class of functions of $p$ halfspaces over $\R^d$, i.e.\ functions $f : \R^d \to \R$ of the form \begin{equation}\label{eq:hs-def}
f(x) = g(\sgn(\inn{w^1, x} + \theta_1), \dots, \sgn(\inn{w^p, x} + \theta_p))
\end{equation} for some $w^1, \dots, w^p \in \R^d$ (where we use superscripts to avoid confusion with coordinate notation), $\theta_1, \dots, \theta_p \in \R$, and $g : \cube{p} \to \cube{}$. We focus on the setting where $p$ is a constant. Also let $\Dx$ be some fixed distribution that is anticoncentrated and $\alpha$-strictly subexponential. We will prove the following theorem, stated earlier as \cref{thm:learn-hs-intro}.

\begin{theorem}\label{thm:learn-hs}
Let $\calC$ be the class of functions of $p$ halfspaces over $\R^d$, as above. Assume that $p = O(1)$. Let $\Dx$ be a distribution that is anticoncentrated and $\alpha$-strictly subexponential. Then $\calC$ can be testably learned w.r.t.\ $\Dx$ up to excess error $\eps$ using $d^{\wt{O}(\eps^{-(1+\alpha)/\alpha})}$ sample and time complexity.
\end{theorem}
In particular whenever $\alpha=1$, as for strongly logconcave distributions including $\calN(0, I_d)$, we obtain a $d^{\wt{O}(1/\eps^2)}$-time algorithm.

We now describe our proof plan. To use \cref{thm:main-algorithm}, we must show that approximately matching the low-degree moments of $\Dx$ fools functions of halfspaces. Work due to \cite{klivans2013moment} introduced an argument for this problem based on  general techniques from the classical theory of moments and the method of metric distances in probability \cite{klebanov1996proximity,rachev2013methods}. Their broad proof approach was to use \cite[Thm 2]{klebanov1996proximity} to show that closeness in moments of two distributions implies closeness in the $\lambda$-distance (\cref{def:lambda-dist}), and then relate this to the CDF distance, which directly relates to fooling halfspaces. For our purposes, a direct application of \cite[Thm 2]{klebanov1996proximity} does not suffice. Instead, we directly analyze the $\lambda$-distance under the assumption that the moments of $\Dx$ grow in a strictly subexponential fashion. We begin with the technical lemmas we need, and then prove \cref{thm:learn-hs} in the final subsection.

\subsection{Moment closeness implies distribution closeness}\label{subsec:lambda-dist}

\begin{definition}[{$\lambda$-distance, see e.g.~\cite{zolotarev1984probability}, \cite[Chap 10]{rachev2013methods}}]\label{def:lambda-dist}
For a distribution $P$ on $\R^p$, let $\phi_P : \R^p \to \mathbb{C}$ given by $\phi_P(t) = \ex_{z \sim P}[e^{i\inn{t, x}}]$ be its characteristic function. For two distributions $P, P'$ on $\R^p$, define the $\lambda$-distance between them as follows: \[ d_{\lambda}(P, P') = \min_{T > 0} \max \left\{ \max_{\|t\| \leq T} \{|\phi_{P}(t) - \phi_{P'}(t) |\}, \frac{1}{T} \right\}. \]
\end{definition}

We prove an approximate version of \cite[Thm 1]{klebanov1996proximity} (see also \cite[Thm 10.3.4]{rachev2013methods}), bounding the $\lambda$-distance between two distributions whose moments approximately match and grow with the degree $k$ in a strictly subexponential fashion, i.e.\ as $k^{k/(1+\alpha)}$.

\begin{lemma}\label{lem:lambda-closeness}
Let $k \in \N$ be even. Let $P$ be a distribution on $\R^p$ such that for all $\|u\| \leq 1$, \[ \ex_{z \sim P}[|\inn{u, z}|^k] \leq M_k := p^{k/2} C_2^k k^{k/(1+\alpha)} \] for some constant $C_2$. Let $P'$ be a distribution that approximately matches moments up to order $k$ with $P$ in the following strong sense: for all $j \leq k$ and $\|u\| \leq 1$, \[ \left| \ex_{z \sim P}[\inn{u, z}^j] - \ex_{z' \sim P'}[\inn{u, z'}^j] \right| \leq \eta_j := \frac{j!}{2k} \left(\frac{6M_k}{k!}\right)^{(j+1)/(k+1)} = \frac{j!}{2k} \left(\frac{\sqrt{p}}{C_4 k^{\alpha/(1+\alpha)}}\right)^{j+1} \] for some constant $C_4$ depending only on $C_2$. Then \[ d_\lambda(P, P') \lesssim \sqrt{p} k^{-\alpha/(1+\alpha)}. \] 
\end{lemma}
\begin{proof}
To control $d_\lambda(P, P')$, we need to control $\max_{\|t\| \leq T} \{|\phi_{P}(t) - \phi_{P'}(t) |\}$ as a function of $T$. To this end, fix any direction $u \in \R^p$ with $\|u\| = 1$, and let $t = \tau u$ for $\tau \in [0,T]$ be a vector in that direction satisfying $\|t\| \leq T$. Let $\phi_1(\tau) = \phi_{P}(\tau u)$ and $\phi_2(\tau) = \phi_{P'}(\tau u)$ be the characteristic functions of $P$ and $P'$ along $u$. We may Taylor expand $\phi_1 - \phi_2$ up to degree $k$ as follows: \begin{equation}\label{eq:char-taylor} \phi_1(\tau) - \phi_2(\tau) = \sum_{0 \leq j < k} \frac{\phi_1^{(j)}(0) - \phi_2^{(j)}(0)}{j!} \tau^j + \frac{\phi_1^{(k)}(\tau') - \phi_2^{(k)}(\tau')}{k!} \tau^k \end{equation} for some $\tau' \in [0, \tau]$.

The crucial fact we use now is that the derivatives of the characteristic function encode its moments. Indeed, for any $\tau$, \[ \phi_1(\tau) = \ex_{z \sim D}[e^{i\tau\inn{z,u}}] \implies \phi_1^{(j)}(\tau) = \ex_{z \sim P}[i^j \inn{z,u}^j e^{i\tau\inn{z,u}}], \] so that in particular $|\phi_1^{(j)}(0)| = |\ex_{z \sim P}[\inn{z,u}^j]|$ for all $j$ (and similarly for $\phi_2$). This means $|\phi_1^{(0)}(0) - \phi_2^{(0)}(0)| = 0$, and for each $1 \leq j < k$, by our assumption that $P'$ approximately moment matches $P$, we have \begin{equation}\label{eq:char0} |\phi_1^{(j)}(0) - \phi_2^{(j)}(0)| \leq \eta_j. \end{equation} At degree $k$, we have \begin{equation}\label{eq:char1} |\phi_1^{(k)}(\tau')| = |\ex_{z \sim P}[i^k \inn{z,u}^k e^{i\tau'\inn{z,u}}]| \leq \ex_{z \sim P}[|\inn{z,u}^k|] \leq M_k.\end{equation} And since $\ex_{z' \sim P'}[\inn{z',u}^k] \leq \ex_{z \sim P}[\inn{z,u}^k] + \eta_k$, we similarly have \begin{equation}\label{eq:char2} |\phi_2^{(k)}(\tau')| \leq M_k + \eta_k \ll 2M_k \end{equation} Substituting \cref{eq:char0,eq:char1,eq:char2} into \cref{eq:char-taylor}, we obtain \begin{equation}\label{eq:char-diff} |\phi_1(\tau) - \phi_2(\tau)| < \sum_{1 \leq j < k} \frac{\eta_j}{j!} \tau^j + \frac{3M_k}{k!}\tau^k =: F(\tau), \end{equation} where we have denoted the expression on the RHS by $F(\tau)$ for convenience. Since $F(\tau)$ is clearly increasing in $\tau$ and independent of $u$, we have $\max_{\|t\| \leq T} \{|\phi_{P}(t) - \phi_{P'}(t) |\} < F(T)$. This means that \[ d_{\lambda}(P, P') \leq \max\{ \max_{\|t\| \leq T} \{|\phi_{P}(t) - \phi_{P'}(t) |\}, \frac{1}{T}\} \leq \max\{F(T), \frac{1}{T}\}, \] and our job now is to pick $T > 0$ that minimizes the RHS.

This is equivalent to picking the largest $T$ such that $F(T) \leq \frac{1}{T}$, i.e.\ \[ TF(T) = \sum_{1 \leq j < k} \frac{\eta_j}{j!} T^{j+1} + \frac{3M_k}{k!}T^{k+1} \leq 1. \] Let us divide this further into two sufficient conditions: \[ \sum_{1 \leq j < k} \frac{\eta_j}{j!} T^{j+1} \leq \frac{1}{2} \quad\text{and}\quad \frac{3M_k}{k!}T^{k+1} = \frac{1}{2}. \] The second condition is equivalent to \[ T = \left( \frac{k!}{6M_k} \right)^{1/(k+1)} = \left( \frac{k!}{6p^{k/2} C_2^k k^{k/(1+\alpha)}} \right)^{1/(k+1)} \eqsim \frac{k^{\alpha/(1+\alpha)}}{\sqrt{p}}, \] by Stirling's approximation. As for the first, we have picked $\eta_j$ exactly such that when we plug in this value of $T$, for each $1 \leq j < k$ we have \[ \eta_j = \frac{j!}{2k} T^{-(j+1)} = \frac{j!}{2k} \left(\frac{6M_k}{k!}\right)^{(j+1)/(k+1)} \implies \frac{\eta_j}{j!} T^{j+1} = \frac{1}{2k}. \] Summing over $1 \leq j < k$ verifies the first condition. Thus for this $T$, we have \[ d_{\lambda}(P, P') \leq \max\{F(T), \frac{1}{T}\} \leq \frac{1}{T} \lesssim \sqrt{p} k^{-\alpha/(1+\alpha)}, \] proving the lemma.
\end{proof}

We offer some remarks to guide the reader through these calculations. For our application, $P$ will be the distribution of $\inn{x,v}$ for $x \sim \Dx$ and $\|v\| = \sqrt{p}$. The key idea is simply to use a Taylor approximation of the $\lambda$-distance to reduce the issue to one of moment closeness. The final calculation amounts to solving for $T$ satisfying $T^{k+1} \eqsim \frac{k!}{M_k}$, where the denominator is the $k\th$ moment of $P$. We then set each $\eta_j$ small enough to make the lower degree terms minor. Note that the $j\th$ moment of $P$ scales as $M_j = p^{j/2} C_2^{j} j^{j/(1+\alpha)}$, and we have $M_k^{j/k} = p^{j/2} C_2^{j} k^{j/(1+\alpha)}$. Thus loosely speaking, the slack $\eta_j$ may be viewed in relative terms as follows: \begin{align} \frac{\eta_j}{M_j} &= \frac{1}{2k} \frac{j!}{M_j} \left(\frac{6M_k}{k!}\right)^{(j+1)/(k+1)} \\
&\approx \frac{1}{2k} \frac{j!}{M_j} \left(\frac{6M_k}{k!}\right)^{j/k} \\
&\approx \frac{1}{2k} \frac{j^j}{p^{j/2} C_2^j j^{j/(1+\alpha)}} \frac{p^{j/2} C_2^j k^{j/(1+\alpha)}}{k^j} \\
&\approx \frac{1}{2k} \left( \frac{j}{k} \right)^{j\alpha/(1+\alpha)}.
\end{align} This relative slack factor is only about $1/\poly(k)$ for $j = O(1)$ but for $j = \Theta(k)$ it becomes $\exp(-\Theta(k))$, which seems unavoidable with our method. Finally, note also that if $P$'s moments scaled only as a subexponential instead of a strictly subexponential distribution, i.e.\ if the $k\th$ moment of $P$ scaled with $k$ as $k^k$, then the key calculation for $T$ becomes vacuous. More involved techniques (see e.g.~\cite[Thm 10.3.1]{rachev2013methods}) still have have something to say in this situation when certain stricter moment conditions hold, but the direct Taylor expansion approach fails.

The following lemma is a convenient distillation of the rest of the argument from \cite{klivans2013moment}, where the $\lambda$-distance is related to the Levy distance (using \cite{gabovich1981stability}), which in turn is related to the CDF distance (using anticoncentration), and which leads finally to the desired conclusion. 

\begin{lemma}[{Implicit in \cite[\S3.3]{klivans2013moment}}]\label{lem:km-lemma}
Let $f : \R^d \to \R$ be a function of $p$ halfspaces as above, and also let  distributions $D, D'$ on $\R^d$ and $P, P'$ on $\R^p$ be as above. Assume that for any continuous interval $T \subset \R$, each coordinate $z_j$ of $z \sim P$ satisfies $\pr[z_j \in T] \leq \Theta(|T|)$. Suppose that $d_{\lambda}(P, P') \leq \delta$. Let $N(\delta)$ be such that $\pr_{z \sim P}[\|z\|_\infty > N(\delta)] \leq \delta$ and $\pr_{z' \sim P'}[\|z'\|_\infty > N(\delta)] \leq \delta$. Then \[ |\ex_D[f] - \ex_{D'}[f]| \leq O \big(2^p \delta \big(\log N(\delta) + 2\log(1/\delta) \big)^p. \]
\end{lemma}

\subsection{Approximate low-degree moment matching fools functions of halfspaces}

We now prove our main structural result, which is that any distribution that approximately matches the low-degree moments of $\Dx$ fools functions of halfspaces.

\begin{theorem}\label{thm:fool-gaussian}
Let $f : \R^d \to \R$ be of the form in \cref{eq:hs-def}. For any $k \in \N$, let $\Delta \in \R_{+}^{|\calI(k,d)|}$ be such that for each $I \in \calI(k,d)$ with $|I| = j$, \begin{equation}\label{eq:delta-def}
\Delta_I = \frac{\sqrt{p}}{2k} \frac{j!}{d^j} \left( \frac{1}{C_4k^{\alpha/(1+\alpha)}} \right)^{j+1}
\end{equation} for some constant $C_4 > 0$. Then for any distribution $D'$ whose moments up to order $k$ match those of $\Dx$ up to $\Delta$, we have \[ |\ex_D[f] - \ex_{D'}[f]| \leq k^{-\alpha/(1+\alpha)} \sqrt{p} \big(C \log (\sqrt{p}k^{\alpha/(1+\alpha)}) \big)^{2p} \] for some constant $C$. In particular, for $p = O(1)$, we have $|\ex_D[f] - \ex_{D'}[f]| \leq \wt{O}(k^{-\alpha/(1+\alpha)})$.
\end{theorem}
\begin{proof}
Let $D = \Dx$. Assume without loss of generality that $w^1, \dots, w^p$ are unit vectors, and let $W \in \R^{p \times d}$ be the matrix with the $w^i$ as its rows. Let $P$ be the distribution (on $\R^p$) of $Wx$ for $x \sim D$, and define $P'$ similarly. We would like to apply \cref{lem:lambda-closeness} to $P$ and $P'$. To do so, we must first verify moment closeness. Let $u \in \R^p$ be a unit vector, and let $v = W^T u \in \R^d$. For any multi-index $I \in \N^d$, let $v_I$ denote $\prod_{j \in [d]} |v_j|^{I_j}$. Then for any $j$, \begin{align}
\ex_{z \sim P}[\inn{z,u}^j] &= \ex_{x \sim D}[\inn{x, v}^j] \\
&= \ex_{x \sim D}[\sum_{|I| = j} x_I v_I] \\
&= \sum_{|I| = j} v_I \ex_{D}[x_I].
\end{align} We place a crude upper bound on each $|v_I|$ as follows. Since $W$ has Frobenius norm $\|W\|_F = \sqrt{p}$, we have \[ \|v\|_\infty \leq \|v\| \leq \|u\| \|W\|_F \leq \sqrt{p}, \] so that in particular for each $I$ with $|I| = j$, $|v_I| = \prod_{j \in [d]} |v_j|^{I_j} \leq \|v\|_\infty^{|I|} \leq p^{j/2}$. Thus \begin{align}
|\ex_{z \sim P}[\inn{z,u}^j] - \ex_{z' \sim P'}[\inn{z',u}^j]| &= |\ex_{x \sim D}[\inn{x, v}^j] - \ex_{x' \sim D'}[\inn{x', v}^j]| \\
&\leq \sum_{|I| = j} |v_I| |\ex_{D}[x_I] - \ex_{D'}[x_I]| \\
&\leq d^j p^{j/2} \sup_{|I| = j} \Delta_I \\
&\leq \eta_j,
\end{align} where $\eta_j$ is as defined in \cref{lem:lambda-closeness}, and the final inequality follows since we have picked $\Delta$ in the theorem statement precisely such that $\sup_{|I| = j} \Delta_I = d^{-j} p^{-j/2} \eta_j$.
Also observe that \[ \ex_{z \sim P}[\inn{u, z}^k] = \ex_{x \sim D}[\inn{x, v}^k] = \|v\|^{k} C_2^k k^{k/(1+\alpha)} \leq p^{k/2} C_2^k k^{k/(1+\alpha)} \] Now we apply \cref{lem:lambda-closeness} to conclude that $d_{\lambda}(P, P') \lesssim \sqrt{p} k^{-\alpha/(1+\alpha)}$.

To finish the proof, we appeal to \cref{lem:km-lemma}. For this we must first verify anticoncentration of $P$ and estimate $N(\delta)$ as defined in that lemma. Observe first that for any $i \in [d]$, the $i\th$ coordinate of $z \sim P$ (resp.\ $z' \sim P')$ is precisely $\inn{w^i, x}$ for $x \sim D$ (resp.\ $\inn{w^i, x'}$ for $x' \sim D'$). Anticoncentration of each coordinate of $z$ follows immediately from \cref{def:strictly-subexp}(c). To estimate $N(\delta)$, we will use a simple Chebyshev-style bound. For any coordinate $i \in [d]$ and any even degree $j \leq k$, we have \[
\pr_{D}[|\inn{w^i, x}| > t] \leq \frac{\ex_D[\inn{w^i, x}^j]}{t^j} \leq \frac{C_2^j j^{j/(1+\alpha)}}{t^j}. \] And since $D'$ approximately matches moments with $D$, by a similar calculation as earlier (now with $\|w^i\|_\infty \leq 1$ in place of $\|v\|_\infty \leq \sqrt{p}$), \[ \ex_{D'}[\inn{w^i, x'}^j] \leq \ex_{D}[\inn{w^i, x}^j] + \sum_{|I| = j} |w_I| \Delta_I \leq \ex_{D}[\inn{w^i, x}^j] + \eta_j / p^{j/2} \ll 2\ex_{D}[\inn{w^i, x}^j], \] and so \[ \pr_{D'}[|\inn{w^i, x'}| > t] \leq \frac{\ex_{D'}[\inn{w^i, x'}^j]}{t^j} \leq \frac{2C_2^j j^{j/(1+\alpha)}}{t^j} = 2 \left(\frac{C_2 j^{1/(1+\alpha)}}{t}\right)^j. \] We need $t$ such that the RHS is at most $\delta/p$. For this it suffices to set $j = 2\log(p/\delta)$ and $t = C_2 j^{1/(1+\alpha)}$ for this. By a union bound over the $p$ coordinates of $z \sim P$ (similarly $z' \sim P'$), we see that we may take $N(\delta) = t = O((\log(p/\delta))^{1/(1+\alpha)})$.

We are now ready to apply \cref{lem:km-lemma} with this $N(\delta)$ and $\delta \eqsim \sqrt{p} k^{-\alpha/(1+\alpha)}$. Substituting these expressions in, we get that \begin{align}
|\ex_D[f] - \ex_{D'}[f]| &\leq O \big(2^p \delta \big(\log N(\delta) + 2\log(1/\delta) \big)^p \big) \\
&\leq 2^p \delta \Big(C' \log(\frac{1}{\delta} \log \frac{p}{\delta}) \Big)^p  \tag*{(for some constant $C' > 0$)} \\
&\leq \delta \Big(C \log(\frac{p}{\delta}) \Big)^{2p}  \tag*{(for some constant $C > 0$)} \\
&\leq k^{-\alpha/(1+\alpha)} \sqrt{p} \big(C \log (\sqrt{p}k^{\alpha/(1+\alpha)}) \big)^{2p},
\end{align} as claimed.
\end{proof}

We pause to note an interesting corollary of this theorem, stated informally earlier as \cref{cor:anticonc-intro}, which states that any $D'$ that approximately matches low-degree moments with $\Dx$ must be anticoncentrated. 
\begin{corollary}\label{cor:anticonc}
Let $\eps > 0$, $k = \wt{O}(\eps^{-(1+\alpha)/\alpha})$, and $\Delta$ be as in \cref{thm:fool-gaussian}, with $p=2$. Let $D'$ be any distribution whose moments up to order $k$ match those of $\Dx$ up to $\Delta$. Then for any $\|u\| = 1$ and any continuous interval $T \subset \R$, $\pr_{x \sim D'}[\inn{x, u} \in T] \leq \pr_{x \sim \Dx}[\inn{x, u} \in T] + \eps \leq \Theta(|T|) + \eps$.
\end{corollary}
\begin{proof}
Write $T = [\theta, \theta']$ for some $\theta < \theta' \in \R$, and consider the function $f(x) = \sgn(\inn{x,u} - \theta) \land \sgn(\theta' - \inn{x, u})$ (where $b_1 \land b_2 = 1$ iff $b_1 = b_2 = 1$). Clearly $\inn{x, u} \in T$ iff $f(x) = 1$. But $f$ is an intersection of two halfspaces, and we know by \cref{thm:fool-gaussian} that $|\ex_{D'}[f] - \ex_{\Dx}[f]| \leq \eps$. Since $\ex_{\Dx}[f] \leq C_3 |T|$ by \cref{def:strictly-subexp}(c), the statement follows.
\end{proof}
In fact, the same reasoning tells us that for any collection of $p = O(1)$ intervals $T$ and any $\|u\| = 1$, $\pr_{D'}[\inn{x,u} \in T] = \pr_{\Dx}[\inn{x,u} \in T] \pm \eps$.

\subsection{Proof of \cref{thm:learn-hs}}
The final ingredient for the proof of \cref{thm:learn-hs} is the following lemma, which gives a bound on the sample complexity required for the empirical moments of $\Dx$ to concentrate about their true moments.
\begin{lemma}\label{lem:moment-conc}
Let the degree parameter be $k$, and the slack parameter $\Delta \in \R_{+}^{|\calI(k,d)|}$ be as in \cref{thm:fool-gaussian}. Assume $p = O(1)$. Then drawing a sample of size $m = d^{\wt{O}(k)}$ from $\Dx$ is sufficient to ensure that with high probability, the empirical moments of order at most $k$ match those of $\Dx$ up to slack $\Delta$.
\end{lemma}
\begin{proof}
Let $D = \Dx$, and let $\hat{D}_m$ denote the empirical distribution on a sample $S$ of size $m$ drawn from $D$. We would like to ensure that for every $I \in \calI(k, d)$, $|\ex_{\hat{D}_m}[x_I] - \ex_{D}[x_I]| \leq \Delta_I$. It suffices to consider the case when $x_I$ has the weakest concentration, and this is clearly when $|I| = k$ and in fact when $I = (k, 0, \dots, 0)$ (without loss of generality), so that $x_I = x_1^k$.

Let $Z$ denote the random variable $\ex_{\hat{D}_m}[x_1^k] - \ex_{D}[x_1^k]$. For our purposes it is sufficient to use a crude Chebyshev bound, although higher moment analogs will give a slightly better bound. We have $\var[Z] = \frac{1}{m} \var[x_1^k] \leq \frac{1}{m} C_2^{2k} (2k)^{2k/(1+\alpha)}$. Thus \begin{align} \pr[Z > \Delta_I] &\leq \frac{\var[Z]}{\Delta_I^2} \\
&\leq \frac{1}{m} \frac{C_2^{2k} (2k)^{2k/(1+\alpha)}}{\Delta_I^2} \\
&\leq \frac{1}{m} C_2^{2k} (2k)^{2k/(1+\alpha)} \left( \frac{2k}{\sqrt{p}} \frac{d^k}{k!} \right)^2 (C_4k^{\alpha/(1+\alpha)})^{2(k+1)} \\
&\leq \frac{1}{m} (dk)^{O(k)},
\end{align} after plugging in the value of $\Delta_I$ when $|I| = k$ from \cref{eq:delta-def} and some manipulation. This is at most $\delta$ if $m \geq (dk)^{O(k)} / \delta$.

For moment closeness to hold simultaneously for all $I \in \calI(k, d)$ with high probability, we set $\delta = \Theta(1/|\calI(k,d)|) = d^{-\Theta(k)}$ and apply a union bound. For this $\delta$, $m$ may be simplified to $d^{\wt{O}(k)}$, as desired.
\end{proof}

We are now ready to prove \cref{thm:learn-hs} using our general algorithm, \cref{thm:main-algorithm}.

\begin{proof}[Proof of \cref{thm:learn-hs}]
We need to pick $m$, $k$ and $\Delta$ suitably as functions of $\eps$ and $d$, and verify that the conditions in \cref{thm:main-algorithm} hold. Let $\Delta$ be as defined in \cref{thm:fool-gaussian}. By \cref{lem:moment-conc}, it suffices to take $m = d^{\wt{O}(k)}$ to ensure that with high probability, the empirical distribution $\hat{D}_m$ matches moments of order at most $k$ with $\calN(0, I_d)$ up to $\Delta$. This verifies condition (a) of \cref{thm:main-algorithm}. Now for any $f \in \calC$, we can combine \cref{thm:fool-gaussian} with \cref{thm:duality} to obtain degree-$k$ sandwiching polynomials satisfying condition (b) of \cref{thm:main-algorithm}, with $\eps = \wt{O}(k^{-\alpha/(1+\alpha)})$, or equivalently $k = \wt{O}(\eps^{-(1+\alpha)/\alpha})$. Applying \cref{thm:main-algorithm} completes the proof.
\end{proof}

\section{Sample complexity of testable learning}

In this section we show that the sample complexity of testably learning a class is characterized by its Rademacher complexity. Throughout this section, let $\Dx$ be the target marginal, and let $\calC$ be the concept class (mapping $\calX$ to $\cube{}$) that we wish to testably learn w.r.t.\ $\Dx$. We remind the reader that we use the term ``with high probability'' to mean ``with probability at least $0.99$'' for simplicity.

\subsection{Upper bound}

We begin with the upper bound, which is essentially just the observation that the empirical Rademacher complexity provides a generalization guarantee that can be estimated to high accuracy from the sample itself. Note that this is an information-theoretic upper bound.

\begin{theorem}\label{thm:upper-bound}
Let $\eps > 0$, and let $m'$ be such that $\calR_{m'}(\calC, \Dx) \leq \frac{\eps}{5}$. Then $\calC$ can be testably learned w.r.t.\ $\Dx$ up to excess error $\eps$ with sample complexity $m' + O(1/\eps^2)$.

\end{theorem}
\begin{proof}
Let $m = m' + O(1/\eps^2)$. Let $S = \{(x_i, y_i)\}_{i \in [m]} \sim \calD^{\otimes m}$ be a sample of $m$ labeled points drawn from $\calD$, and let $S_\calX$ denote $\{x_i\}_{i \in [m]}$. Our tester $T$ accepts iff $\hat{\calR}_m(\calC, S_{\calX}) \leq \frac{\eps}{4}$. Whenever the tester accepts, the learner $A$ simply performs ERM over $S$ w.r.t.\ $\calC$.

To see why completeness is satisfied, suppose that the true marginal is in fact $\Dx$. Since $\calR_{m}(\calC, \Dx) \leq \calR_{m'}(\calC, \Dx) \leq \frac{\eps}{5}$, by \cref{eq:emp-rad-conc} we can ensure that with high probability over $S_{\calX}$, $\hat{\calR}_m(\calC, S_\calX) \leq \frac{\eps}{4}$, and so $T$ will accept.

Soundness holds by a standard argument showing generalization using uniform convergence. Formally, suppose that the tester accepts $S_{\calX}$, i.e.\ $\hat{\calR}_m(\calC, S_{\calX}) \leq \frac{\eps}{4}$. Consider an ERM hypothesis \[ \hat{f}_m = \argmin_{f \in \calC} \frac{1}{m} \sum_{i \in [m]} \ell(f(x_i),y_i) \] as well as an optimal hypothesis \[ f^* = \argmin_{f \in \calC} \ex_{(x,y) \sim \calD}[\ell(f(x),y)]. \] Then with high probability over $S$ we have \begin{align}
L(\hat{f}_m) &\leq \hat{L}_m(\hat{f}_m) + \hat{\calR}_m(\calC) + O \left(\sqrt{\frac{1}{m}}\right) \tag*{(by \cref{eq:rad-comp-data})} \\
&\leq \hat{L}_m(f^*) + \hat{\calR}_m(\calC) + O \left(\sqrt{\frac{1}{m}}\right) \tag*{($\hat{f}_m$ is an ERM hypothesis)} \\
&\leq L(f^*) + 2\hat{\calR}_m(\calC) + O \left(\sqrt{\frac{1}{m}}\right) \tag*{(by \cref{eq:rad-comp-data} again)} \\
&\leq L(f^*) + \frac{\eps}{2} + \frac{\eps}{2} = L(f^*) + \eps,
\end{align} by our choice of $m$. This proves the theorem.
\end{proof}

\subsection{Lower bound}

Now we state the lower bound, which matches the upper bound up to a quadratic factor. Our lower bound can be viewed as a generalization of the argument of \cite{rubinfeld2022testing}, who proved lower bounds for testable-learning for the specific cases of convex sets and monotone functions.  We obtain the full range of lower-bounds for any value of $\epsilon$ purely in terms of Rademacher complexity. The idea here is that no tester with bounded sample complexity $m$ can distinguish between a distribution $\Dx$ and the uniform distribution on a sufficiently large sample of size $M \gg m$ drawn from $\Dx$. If the Rademacher complexity w.r.t. $\Dx$ at sample size $M$ is somewhat large, then the large sample can be labeled randomly and still admit nontrivial optimal error, but of course the learner cannot do well on unseen data.

\begin{theorem}\label{thm:lower-bound}
Let $\eps > 0$, and let $M$ be such that $\calR_M(\calC, \Dx) \geq 5\eps$. Assume that $M \geq \Theta(1/\eps^2)$ is sufficiently large and that a sample of size $M$ drawn from $\Dx$ will, with high probability, contain no duplicates. Then testably learning $\calC$ up to excess error $\eps$ requires sample complexity at least $\Omega(\sqrt{M})$.
\end{theorem}
\begin{proof}
Suppose we had a tester-learner $(T, A)$ requiring sample complexity only $m$ where $m \leq \frac{\sqrt{M}}{100}$. We will show how to ``fool'' $(T, A)$ into failing its guarantee by constructing a labeled distribution $\calD$ such that \begin{enumerate}[(a)]
\item $\opt(\calD, \calC) \leq \frac{1}{2} - 2\eps$;
\item with high probability, $T$ will accept a sample of size $m$ drawn from $\calD$; and yet
\item with high probability, $A$'s output will have error greater than $\frac{1}{2} - \eps$ on $\calD$.
\end{enumerate}
For such a $\calD$, it is clear that the tester-learner pair $(T, A)$ fails its guarantee in that with high probability, despite $T$ accepting, $A$ cannot produce a hypothesis with error at most $\opt(\calD, \calC) + \eps$.

We construct $\calD$ as follows. Draw a sample of $M$ randomly labeled points $S = \{(x_i, y_i)\}_{i \in [M]} \sim (\Dx \times \unif\cube{})^{\otimes M}$, and let $S_\calX$ denote $\{x_i\}_{i \in [m]}$. Define $\calD$ to be the uniform distribution over $S$. We now show that with high probability over the draw of $S$ (including its random labeling), the distribution $\calD$ satisfies the required properties. 

First, condition on $S_\calX$ containing no duplicates, which occurs with high probability by assumption. Denote the size-$m$ sample given to $(T, A)$ by $S' \sim \calD^{\otimes m}$, and let $S'_\calX$ denote its marginal. Let us also condition on $S'_\calX \subset S_\calX$ containing no duplicates, which occurs with high probability since the probability of duplicates in $S'_\calX$ (as $S_\calX$ itself contains no duplicates) is at most $m^2/M \leq 10^{-4}$.

Let us see why property (a) holds with high probability over $S$. The idea is that because $\calR_M(\calC) \geq \Omega(\eps)$, we expect that there exists a classifier in $\calC$ that achieves error at most $\frac{1}{2} - \Omega(\eps)$ on the randomly labeled sample $S$. Formally, observe that since $S_\calX$ contains no duplicates, the random labels are exactly equivalent to Rademacher random variables. Assuming that $M$ is sufficiently large and applying \cref{eq:emp-rad-conc,,eq:emp-rad-conc-2} successively, we obtain that with high probability over the sample $S$ (together with the realization of the random labels), \[ \left|\calR_M(\calC) - \sup_{f \in \calC} \frac{1}{M} \sum_{i=1}^M y_i f(x_i) \right| \leq \eps. \] In particular, since $\calR_M(\calC) \geq 5\eps$, there exists $f^* \in \calC$ such that $\frac{1}{M} \sum_{i=1}^M y_i f^*(x_i) \geq 4\eps$, or equivalently \[  \frac{1}{M} \sum_{i=1}^M \ind[f^*(x_i) \neq y_i] = \frac{1}{M} \sum_{i=1}^M \frac{1 - y_i f^*(x_i)}{2} = \frac{1}{2} - \frac{1}{2M} \sum_{i=1}^M y_i f^*(x_i) \leq \frac{1}{2} - 2\eps. \] In other words, $\opt(\calD, \calC) \leq \frac{1}{2} - 2\eps$.

Property (b) is straightforward since the marginal that the tester observes is entirely consistent with $\Dx$: because we have conditioned on $S'_\calX$ containing no duplicates, $S'_\calX$ is distributed exactly as a sample of $m$ points drawn directly from $\Dx$. Thus any tester satisfying completeness must accept $S'$ with high probability. 

For property (c), the idea is that the learner, having only seen a minuscule fraction of the randomly labeled $\calD$, cannot possibly output a hypothesis with error substantially better than $\frac{1}{2}$ on all of $\calD$. Formally, observe that any classifier $h$ that $A$ outputs is stochastically independent of $S \setminus S'$. This means that in expectation over $S$ and the randomness of $A$, \[ \pr_{(x,y) \sim \calD}[h(x) \neq y] \geq 0 \cdot \frac{m}{M} + \frac{1}{2} \cdot \frac{M-m}{M} = \frac{1}{2} - \frac{m}{2M} \geq \frac{1}{2} - \frac{\eps}{2}, \] since the fact that $M \geq \Theta(1/\eps^2)$ and $m \leq \sqrt{M}/100$ mean that $\frac{m}{M} \leq \frac{1}{100\sqrt{M}} < \eps$. Clearly for sufficiently large $M \geq \Theta(1/\eps^2)$, with high probability we will have $\pr_{(x,y) \sim \calD}[h(x) \neq y] > \frac{1}{2} - \eps$. (Such an argument is also formalized as \cite[Lemma 25]{rubinfeld2022testing}.)

Since properties (a), (b), (c), as well as the property of containing no duplicates, each hold with high probability over $S$, we conclude that there does exist an $S$ such that $\calD$ satisfies all three properties and hence fools $(T, A)$.
\end{proof}

Note that this theorem becomes stronger if $\eps$ is taken to be a constant. In particular, if we assume $\calR_M(\calC, \Dx) \geq 0.99$, then the same argument would actually yield a ``fooling distribution'' $\calD$ such that \begin{enumerate}[(a)]
\item $\opt(\calD, \calC) \leq 0.01$;
\item with high probability, $T$ will accept a sample of size $m$ drawn from $\calD$;
\item with high probability, $A$'s output will have error greater than $0.49$ on $\calD$.
\end{enumerate} This would rule out any tester-learner capable of testably learning up to error sufficient to distinguish the case where $\opt(\calD, \calC) = 0.01$ from $\opt(\calD, \calC) > 0.49$ (e.g., one with final error guarantee $10 \cdot \opt(\calD, \calC) + 0.1$).

We also give the following stronger version of this lower bound, stated in terms of the behavior of the empirical Rademacher complexity (which is a random variable depending on the sample). This is a very strong lower bound that holds whenever $\hat{\calR}_M(\calC, S_\calX) = 1$ with high probability, because it yields a fooling distribution $\calD$ that is in fact perfectly realizable. As we will see in \cref{subsec:apps}, this turns out to apply to convex sets and monotone functions. In a sense, this version is not really about Rademacher complexity but rather the stronger notion of shattering (except with high probability over a sample, like a distribution-specific version of the VC dimension, albeit stronger than VC entropy). Recall that $\calC$ is said to shatter an unlabeled set $S_\calX$ if every possible labeling of $S_\calX$ can be achieved by some $f \in \calC$, or equivalently $\hat{\calR}_M(\calC, S_\calX) = 1$.

\begin{theorem}\label{thm:lower-bound-strong}
Let $M$ be such that with high probability over a size-$M$ sample $S_\calX \sim \Dx^{\otimes M}$, $\hat{\calR}_M(\calC, S_\calX) = 1$, i.e.\ $\calC$ shatters $S_\calX$. Consider any tester-learner pair $(T, A)$ for testably learning $\calC$ up to error sufficient to distinguish the case where $\opt(\calD, \calC) = 0$ from $\opt(\calD, \calC) > 0.49$. Then $(T, A)$ requires sample complexity at least $\Omega(\sqrt{M})$.

In particular, this rules out any tester-learner with final error guarantee $\psi(\opt(\calD, \calC)) + 0.49$ for any increasing function $\psi : [0,1] \to \R$ satisfying $\psi(0) = 0$.
\end{theorem}
\begin{proof}
The proof is a simpler version of the earlier one. Again, suppose we had a tester-learner $(T, A)$ requiring sample complexity only $m \leq \frac{\sqrt{M}}{100}$. We construct a labeled distribution $\calD$ such that \begin{enumerate}[(a)]
\item $\opt(\calD, \calC) = 0$;
\item with high probability, $T$ will accept a sample of size $m$ drawn from $\calD$; and yet
\item with high probability, $A$'s output will have error greater than $0.49$ on $\calD$.
\end{enumerate}
The distribution $\calD$ is constructed in exactly the same way: draw a sample of $M$ randomly labeled points $S = \{(x_i, y_i)\}_{i \in [M]} \sim (\Dx \times \unif\cube{})^{\otimes M}$, and define $\calD$ to be the uniform distribution over $S$. Let $S_\calX = \{x_i\}_{i \in [M]}$. Let $S' \sim \calD^{\otimes m}$ denote the sample given to $(T, A)$, and as before, let us condition on its marginal $S'_\calX$ containing no duplicates (which occurs with high probability).

Property (a) follows immediately from our assumption that with high probability, $\calC$ shatters $S_\calX$. (Note that this also implies that $S_\calX$ contains no duplicates.) Properties (b) and (c) follow by almost exactly the same arguments as before (for the latter, we now use the fact that $m/M \ll 1/100$ instead of $m/M \leq \eps$).
\end{proof}

\subsubsection{Applications}\label{subsec:apps}

The lower bounds of \cite{rubinfeld2022testing} may be viewed as applications of \cref{thm:lower-bound-strong}. The first application is the class of convex sets w.r.t. $\calN(0, I_d)$, and the second is the class of monotone Boolean functions w.r.t.\ $\unif\cube{d}$.
\begin{theorem}[Implicit in \cite{rubinfeld2022testing}, Theorem 22]\label{thm:shatter-conv-sets}
Let $\calX = \R^d$, $\Dx = \calN(0, I_d)$, and $\calC$ be the class of ($\cube{}$-valued indicator functions of) convex sets in $\R^d$. Let $M = 2^{Cd}$ for some small constant $C > 0$. Then with probability $1 - \exp(-\Omega(d))$ over the draw of a size-$M$ sample $S \sim \Dx^{\otimes M}$, $\calC$ can shatter $S$.
\end{theorem}
\begin{theorem}[Implicit in \cite{rubinfeld2022testing}, Theorem 23]\label{thm:shatter-mono}
Let $\calX = \cube{d}$, $\Dx = \unif\cube{d}$, and $\calC$ be the class of monotone Boolean functions. Let $M = 2^{Cd}$ for some small constant $C > 0$. Then with probability $1 - \exp(-\Omega(d))$ over the draw of a size-$M$ sample $S \sim \Dx^{\otimes M}$, $\calC$ can shatter $S$.
\end{theorem}

In fact, Rubinfeld and Vasilyan are able to state their lower bounds in a slightly stronger way because of the specific parameters $M, m$ that these examples above allow. Specifically, for both convex sets and monotone functions, we may take $M = 2^{\Omega(d)}, m = M^{0.01} = 2^{\Omega(d)}$, and the same argument as in \cref{thm:lower-bound-strong} can be analyzed more closely to yield a distribution $\calD$ such that $\opt(\calD, \calC) = 0$ and yet the final output of any testable learner with sample complexity $m$ must have $\exp(-\Omega(d))$ advantage over random guessing (which is stronger than merely saying the output must have error at least $0.49$).

Interestingly, these examples add to what has been called ``the emerging analogy between symmetric convex sets in Gaussian space and monotone Boolean functions''; see \cite{de2022convex} and references therein.

\section{Discussion}

\subsection{Implications for the uniform convergence paradigm}\label{subsec:unif-cov}
As observed in \cite{rubinfeld2022testing}, an interesting consequence of the lower bounds in \cref{subsec:apps} is that they demonstrate a strict separation between distribution-specific agnostic learning and testable learning. In the case of both convex sets over $\calN(0, I_d)$ and monotone functions over $\unif\cube{d}$, Fourier-theoretic arguments are known to give agnostic learners requiring sample complexity only $2^{\wt{O}(\sqrt{d}/\poly(\eps))}$ to learn up to excess error $\eps$ \cite{bshouty1996fourier,klivans2008learning}. In particular, they require only sample complexity $2^{\wt{O}(\sqrt{d})}$ to learn up to excess error $\eps = 0.1$ (say), which is much smaller than the lower bounds of $2^{\Omega(d)}$ for testably learning these classes up to $\eps = 0.1$.

But we have just characterized testable learning in terms of Rademacher complexity, which we know in turn tightly characterizes uniform convergence (\cref{thm:rad-unif-conv}). We draw the following implications from this: \begin{itemize}
\item Uniform convergence is always sufficient for distribution-specific agnostic learning but it is not necessary, as witnessed by the examples of convex sets and monotone functions.
\item Uniform convergence is both necessary and sufficient for testable learning.
\end{itemize}
That is, not only is there a strict separation between distribution-specific agnostic learning and testable learning, it is the latter that is in fact characterized by uniform convergence.

\paragraph{Uniform convergence in distribution-free vs distribution-specific learning.} In the distribution-free setting, uniform convergence is well-known to be necessary and sufficient for agnostic (as well as realizable) learning, by classic VC theory (see e.g.\ \cite[Chapter 6]{shalev2014understanding}). Let us clarify that in the distribution-free setting the term ``uniform convergence'' now means a uniform bound on the generalization gap over not just all $f \in \calC$ but also all distributions $\Dx$; that is, we now care about the \emph{worst-case distribution-free} generalization gap: $\sup_{\Dx} \sup_{f \in \calC} |L(f) - \hat{L}_m(f)|$. This quantity is tightly governed by the VC-dimension of $\calC$, a distribution-free, purely combinatorial property.

Meanwhile in the distribution-specific setting, the statistical complexity of agnostic learning is known to be characterized by the metric entropy of the class (aka the log covering number, w.r.t.\ the metric $\rho(f,g) = \pr_{x \sim \Dx}[f(x) \neq g(x)]$) \cite{benedek1991learnability}. In this setting, uniform convergence (now in the distribution-specific sense of bounds on $\sup_{f \in \calC} |L(f) - \hat{L}_m(f)|$) is sufficient but not necessary. Indeed, we may also view the separations given by convex sets and monotone functions as separations between the metric entropy and the Rademacher complexity of these classes.

A priori, this seems like a surprising difference between distribution-free and distribution-specific agnostic learning. However, one could argue that the more realistic distribution-specific supervised learning model is that of testable learning. Here we see that uniform convergence is again necessary and sufficient.

\paragraph{Relationship to modern overparametrized models.} The inadequacies of the uniform convergence paradigm have been a topic of much study in modern deep learning theory (see e.g.\ \cite{zhang2021understanding,nagarajan2019uniform,bartlett2021deep,belkin2021fit}). We may phrase the essential argument in the following way. Let $\calC$ be a certain ``rich'' concept class mapping $\calX$ to $\cube{}$ (for concreteness). Let $\calD$ be an unknown labeled distribution on $\calX \times \cube{}$, and let $S \sim \calD^{\otimes m}$ be a sample drawn from it. Let $L$ and $\hat{L}_m$ denote the population and empirical 0-1 loss functionals, as before. Consider an ERM estimator $\hat{f}$ picked based on this sample: $\hat{f} \in \argmin_{f \in \calC} \hat{L}_m(f)$. We are interested in the generalization gap associated with $\hat{f}$, namely the quantity $ |L(\hat{f}) - \hat{L}_m(\hat{f})|$. We would like to place a useful upper bound, say $B$, on this quantity.

The core observation is that certain classes $\calC$ (such as deep neural networks) are rich enough that they can interpolate any sample of size $m$; in this sense they are ``overparametrized'' relative to sample size $m$. In particular, they can fit even completely random labels. This of course means that $\hat{L}_m(\hat{f}) = 0$ while $L(\hat{f}) = \frac{1}{2}$. This in turn means that the bound $B$ must be at least $\frac{1}{2}$. Note that this occurs without changing anything about the class $\calC$, the marginal distribution $\Dx$ of $\calD$, or the training procedure (ERM). So any bound $B$ that is purely a function of these quantities must be essentially vacuous; this includes uniform convergence bounds (e.g., $\sup_{f \in \calC} |L(f) - \hat{L}_m(f)| \eqsim \calR_m(\calC, \Dx)$) as well as algorithm-based bounds (e.g.\ those based on stability). Yet what is remarkable is that when the labels do satisfy some structure, e.g.\ when there exists $f^* \in \calC$ achieving error $L(f^*) = \opt(\calC, \calD) < \frac{1}{2}$, then we observe (provably or empirically) that the generalization gap is in fact relatively small, and $\hat{f}$ performs comparably with $f^*$. This phenomenon, sometimes referred to as ``benign overfitting'', occurs not only with deep neural networks but in fact also already (in a sense) with linear regression \cite{bartlett2020benign, hastie2022surprises}; we shall not attempt a summary of known results here but direct the reader to e.g.\ \cite{bartlett2021deep,belkin2021fit}.

What the results in this paper point out is that a version of this phenomenon also occurs in a strong, provable sense with classical examples such as convex sets in Gaussian space or monotone functions over the Boolean hypercube. These classes $\calC$ are also capable of interpolating a random sample of size $m = 2^{\Theta(d)}$; and yet there exist estimators $\hat{f}$ that achieve error $\opt(\calC, \calD) + \eps$ using sample complexity only $2^{\wt{O}(\sqrt{d})/\poly(\eps)}$ \cite{bshouty1996fourier,klivans2008learning}. These estimators are not based on ERM and do not lie strictly in $\calC$; instead, they are low-degree (specifically, degree-$O(\sqrt{d}/\poly(\eps)$) polynomial approximators of functions in $\calC$ (as in \cref{thm:kkms}). Such polynomial approximators essentially constitute a small cover of the class $\calC$ (w.r.t.\ the metric $\rho_1(f, p) = \ex[|f - p|]$). The improved sample complexity we obtain by such methods may be explained by the fact that to obtain generalization, we only require uniform convergence over this cover as opposed to all of $\calC$ (as in the metric entropy characterization of \cite{benedek1991learnability}).



\subsection{Implications for sandwiching degree}

Another somewhat surprising consequence of the lower bounds in \cref{subsec:apps} is that the classes of convex sets over $\calN(0, I_d)$ and monotone functions over $\unif\cube{d}$ cannot admit sandwiching polynomials of degree $o(d/\log d)$ and error even $\eps = \Theta(1)$ unless they have very large coefficients. This is simply because any such sandwiching polynomials, if they have reasonable coefficients and if the distribution satisfies some concentration properties, will tend to allow the moment matching algorithm (\cref{thm:main-algorithm}) to succeed. More generally, we obtain the following surprising connection between Rademacher complexity and sandwiching degree as a direct corollary of \cref{thm:main-algorithm,,thm:lower-bound}.

\begin{corollary}
Let $\eps > 0$, let $\Dx$ be a distribution on $\calX$, and let $\calC$ be a concept class mapping $\calX$ to $\cube{}$. Let $M$ be such that $\calR_M(\calC) \geq 5\eps$, and assume $M \geq \Theta(1/\eps^2)$. Consider any degree and slack parameters $k \in \N$, $\Delta \in \R_{+}^{|\calI(k,d)|}$ such that each $f \in \calC$ admits degree-$k$ sandwiching polynomials $p_l \leq f \leq p_u$ satisfying \[ \ex_{\Dx}[p_u - f] + \inn{\Delta, |p_u|} \leq \frac{\eps}{2}, \qquad \ex_{\Dx}[f - p_l] + \inn{\Delta, |p_l|} \leq \frac{\eps}{2}. \] Let $m$ be the sample complexity of testing with high probability whether the degree-$k$ empirical moments of $\Dx$ are within $\Delta$ of their true moments. Then we must have $m \geq \Omega(\sqrt{M})$.
\end{corollary}

Let us illustrate this in the cases where $\calC, \Dx$ are either convex sets over $\calN(0, I_d)$ or monotone functions over $\unif\cube{d}$. Consider any degree and slack parameters $k, \Delta$, and let $\delta = \min_{I \in \calI(k,d) \setminus \{0\}} \Delta_I$. In both cases, one can check with high probability whether the degree-$k$ empirical moments of $\Dx$ are within $\Delta$ of their true moments using sample complexity at most $m \leq d^{O(k)}\poly(1/\delta)$ (for $\unif\cube{d}$ this is immediate by boundedness, while for $\calN(0, I_d)$ we appeal to \cref{lem:moment-conc}). Now suppose that each $f \in \calC$ admitted degree-$k$ $(\eps/4)$-sandwiching polynomials $p_l \leq f \leq p_u$ satisfying $\ex_{\Dx}[f - p_l], \ex_{\Dx}[p_u - f] \leq \eps/4$ and also such that their coefficients are bounded in magnitude by $d^{O(k)}$. Then clearly we can pick $\Delta$ sufficiently small so that $\inn{\Delta, |p_l|}, \inn{\Delta, |p_u|} \leq \eps/4$ while still ensuring $\delta \geq \eps d^{-O(k)}$. This means $m$ as defined earlier is $d^{O(k)}\poly(1/\eps)$. Thus for this choice of $\Delta$, both conditions (a) and (b) of \cref{thm:main-algorithm} hold, and we obtain a testable learning algorithm with sample complexity $m = d^{O(k)}\poly(1/\eps)$. For $\eps = 0.1$, say, we know by \cref{subsec:apps} that the required sample complexity for this task is $2^{\Omega(d)}$. Thus we see that $k$ must necessarily be $\Omega(d / \log d)$.

The only way for sandwiching polynomials to exist despite this obstacle is by having unusually large coefficients (on the scale of $d^{\omega(k)}$). Most reasonable approaches to constructing sandwiching polynomials will tend to ensure some boundedness of coefficients (indeed, this is true whenever such polynomials are constructed out of univariate polynomials that are bounded on a bounded domain, see e.g.\ \cite[Lemma 4.1]{sherstov2012making}). Therefore, we regard this as good evidence in favor of a lower bound on the sandwiching degree for these classes.

\bibliographystyle{alpha}
\bibliography{custom,refs}

\appendix

\section{Proof of strong duality in \cref{thm:duality}}\label{sec:formal-duality}

We will use the following statement of conic duality, specialized to the setting of moment problems.

\begin{theorem}[{\cite[Section 3]{shapiro2001duality}}]\label{thm:conic-duality}
Let $\Omega = \R^d$, endowed with the standard Borel sigma algebra, and let $\calC$ be the set of all nonnegative Borel measures on $\Omega$. Pair the space of signed measures on $\Omega$ and functions mapping $\Omega$ to $\R$ using the following inner product: $\inn{g, \mu} = \int_{\Omega} g \diff\mu$. Let $\phi, \psi_1, \dots, \psi_p : \Omega \to \R$ be functions, let $b \in \R^p$, let $A : \mu \mapsto (\inn{\psi_1, \mu}, \dots, \inn{\psi_p, \mu})$, and let $K$ be a closed convex cone in $\R^p$.

Define the following primal problem (\cite[Eq 3.2]{shapiro2001duality}): \begin{equation}\label{eq:shap-primal}
\sup_{\mu \in \calC}\ \inn{\phi, \mu} \quad \text{subject to} \quad A\mu - b \in K.
\end{equation} Let $K^* = \{ \alpha \mid \alpha \cdot \alpha' \geq 0\ \forall \alpha' \in K \}$ be the polar cone of $K$. Then the dual is defined as follows (\cite[Eq 3.8]{shapiro2001duality}): \begin{equation}\label{eq:shap-dual}
\inf_{\alpha \in -K^*}\ b \cdot \alpha \quad \text{subject to} \quad \sum_{i=1}^{p} \alpha_i \psi_i(\omega) \geq \phi(\omega)\ \forall \omega \in \Omega.
\end{equation} Further, a sufficient condition for strong duality to hold (i.e.\ for both primal and dual to have the same finite optimum) is that $b$ lie in the interior of the feasible set, i.e.\ $b \in \{\wt{b} \mid \exists \mu \in \calC : A\mu - \wt{b} \in K \}$ (\cite[Eq 3.12]{shapiro2001duality}).
\end{theorem}


Let $\Omega, \calC$ and the dual pairing $\inn{\cdot, \cdot}$ be as above. Note that $\calC$ can also be viewed as the convex cone generated by all Dirac measures on $\Omega$, and also that when $\mu$ is a probability measure (i.e., nonnegative and with total measure 1), $\inn{g, \mu} = \ex_{\mu}[g]$.

Our goal now is to obtain strong duality between \cref{eq:lp-primal,eq:lp-dual} as a consequence of \cref{thm:conic-duality}. Let $r = |\calI(k,d)|-1$, and for convenience write $\calI(k,d) = \{I_0, I_1, \dots, I_r\}$, where $I_0 = (0, \dots, 0)$. Define the functions $\psi_1, \dots, \psi_r : \Omega \to \R$ to be the nontrivial monomials corresponding to $\calI(k,d)$, i.e., $\psi_j(x) = x_{I_j}$, and define $\psi_{r+j} = -\psi_r$ for all $1 \leq j \leq r$. Let $K \subset \R^{2r + 1}$ be the following convex cone: $K = \{0\} \times \R_{-}^{2r}$, where $\R_{-} = (-\infty, 0]$. Let $A$ be a linear operator on $\calC$ given by \begin{align}
A\mu &= (\inn{1,\mu}, \inn{\psi_1, \mu},\dots, \inn{\psi_r, \mu}, \inn{-\psi_{1}, \mu}, \dots, \inn{-\psi_{r}, \mu}) \intertext{and let $b \in \R^{2r+1}$ be given by}
b &= (1, \sigma_{I_1} + \Delta, \dots, \sigma_{I_r} + \Delta, -\sigma_{I_1} + \Delta, \dots, -\sigma_{I_r} + \Delta).
\end{align} We claim that our original primal LP \cref{eq:lp-primal} corresponds to the following conic linear program, which has the form of \cref{eq:shap-primal}, with $p=2r+1$ and $\phi = f$: \begin{equation}
\sup_{\mu \in \calC}\ \inn{f, \mu} \quad \text{subject to} \quad A\mu - b \in K
\end{equation} Indeed, the first coordinate of $b$ ensures the $I = 0$ constraint, namely that $\inn{1, \mu} = 1$ and hence $\mu$ is a valid probability measure (note that the cone $\calC$ already only consists of nonnegative measures), and the other coordinates ensure that $\sigma_I - \Delta_I \leq \ex_{\mu}[x_I] \leq \sigma_I + \Delta_I$ for every other $I \in \calI(k,d) \setminus \{0\}$.

The dual of this program may be written in the form of \cref{eq:shap-dual} as follows. First introduce dual variables $\alpha_0 \in \R$ (corresponding to the first constraint), and $(\alpha_1, \dots, \alpha_{2r}) \in \R^{2r}$ (corresponding to the others), and write $b = (b_0, \dots, b_{2r})$. The dual is \begin{equation}
\inf_{\alpha \in -K^*} \alpha_0 b_0 + \sum_{j=1}^{2r} \alpha_j b_j \quad \text{subject to} \quad \alpha_0 + \sum_{j=1}^{2r} \alpha_j \psi_j \geq f \ \text{over } \Omega.
\end{equation} Here $K^*$ is the polar cone of $K$, and is easily seen to be $K^* = \R \times \R_{-}^{2r}$. This means $-K^* = \R \times \R_+^{2r}$, i.e.\ $\alpha_0 \in \R$ and $(\alpha_1, \dots, \alpha_{2r}) \in \R_{+}^{2r}$. The dual objective may be simplified as follows: \begin{align}
\alpha_0 b_0 + \sum_{j=1}^{2r} \alpha_j b_j &= \alpha_0 + \sum_{j=1}^{r}\left( \alpha_j (\sigma_{I_j} + \Delta_{I_j}) + \alpha_{r+j} (-\sigma_{I_j} + \Delta_{I_j}) \right) \\
&= \alpha_0 + \sum_{j=1}^r (\alpha_j - \alpha_{r+j}) \sigma_{I_j} + \sum_{j=1}^r (\alpha_j + \alpha_{r+j}) \Delta_{I_j}.
\end{align} The constraint simplifies to \[ \alpha_0 + \sum_{j=1}^{r}(\alpha_j - \alpha_{j+r})\psi_j \geq f. \] To simplify this further, if we let $\beta_j = \alpha_j - \alpha_{j+r}$ for every $1 \leq j \leq r$, then it is not hard to see that the objective is minimized when each $\alpha_j + \alpha_{j+r} = |\beta_j|$ (in particular, when $\alpha_j = \max\{\beta_j, 0\}$ and $\alpha_{r+j} = \max\{-\beta_j, 0\}$). Thus if we also let $\beta_0 = \alpha_0$, then the dual objective becomes $\beta_0 + \sum_{r=1}^j \beta_j \sigma_{I_j} + \sum_{j=1}^r|\beta_j| \Delta_{I_j}$, and the constraint becomes $\beta_0 + \sum_{j=1}^r \beta_j \psi_j \geq f$. Recalling that $\sigma_{I_0} = 1$ and $\Delta_{I_0} = \Delta_0 = 0$, this is precisely the dual we originally claimed, \cref{eq:lp-dual}.

Now, by \cref{thm:conic-duality}, a sufficient condition for strong duality is that $b$ lie in the interior of the feasible set, i.e.\ $b \in \{\wt{b} \mid \exists \mu \in \calC : A\mu - \wt{b} \in K \}$. This means that for any sufficiently small perturbation $\wt{b}$ of $b$, there must exist a measure $\mu \in \calC$ such that $A\mu - \wt{b} \in K$, i.e.\ with $\wt{b}$ as its approximate vector of moments up to order $k$. We argue this slightly informally as follows. Let $\mu^*$ denote $D$ from the statement of \cref{thm:duality}. Suppose \begin{align}
\wt{b} &= b + \eta = (b_0 + \eta_0, b_1 + \eta_1, \dots, b_{r+1} +\eta_{r+1}, \dots) \\
&= (1 + \eta_0, \sigma_{I_1} + \Delta_{I_1} + \eta_1, \dots, -\sigma_{I_1} + \Delta_{I_1} + \eta_{r+1}, \dots),
\end{align}
where $\eta_0, \dots, \eta_{2r} \in \R^{2r+1}$ are to be thought of as small. The condition that $A\mu - \wt{b} \in K$ is the same as saying that $\mu$ satisfies the following: \begin{align}
&\inn{\mu, 1} = 1 + \eta_0 \\
\sigma_{I_j} - \Delta_{I_j} - \eta_{r+j} \leq &\inn{\mu, \psi_j} \leq \sigma_{I_j} + \Delta_{I_j} + \eta_{j} \quad \forall 1 \leq j \leq r.
\end{align} For sufficiently small $\eta$, we claim that a small perturbation of $\mu^*$ will continue to satisfy these conditions. First, note that because $\inn{\mu, 1} \neq 1$, $\mu$ is no longer formally a probability measure. But for sufficiently small $\eta_0$, by adding or removing some mass to $\mu^*$ arbitrarily close to the origin, we can increase or decrease its total mass while keeping all its moments nearly unchanged (because the $\psi_j$ are continuous and $\psi_j(0) = 0$ for all $j \neq 0$, and the new mass is essentially all at $0$). Take $\mu$ to be such a perturbation of $\mu^*$, satisfying $\inn{\mu, 1} = 1 + \eta_0$. We have just argued that for every $j \neq 0$, $\inn{\mu, \psi_j}$ differs from $\inn{\mu^*, \psi_j} = \sigma_{I_j}$ by an arbitrarily small amount. Thus if $\eta_1, \dots, \eta_{2r}$ are sufficiently small (it suffices to have each $\eta_j \leq \Delta_{I_j}/2$), then the approximate moment matching conditions will still be satisfied by $\mu$, because there is still a slack of at least $\Delta_{I_j}/2 > 0$ in the constraint arising from $I_j$. This establishes that $b$ is indeed in the interior of the feasible set, and hence that strong duality holds between \cref{eq:lp-primal} and \cref{eq:lp-dual}.
 
\end{document}